\documentclass[11pt, a4paper]{article}

\usepackage[a4paper]{geometry}

\usepackage[english]{babel}

\usepackage[utf8]{inputenc}
\usepackage{xspace}
\usepackage{amsmath,amsthm,amssymb,mathtools}\usepackage{lmodern}
\usepackage{bbm}
\usepackage[numbers]{natbib}
\usepackage{algorithm}
\usepackage[noend]{algpseudocode}
\usepackage{xcolor}
\usepackage{tikz}
\usepackage{graphicx}
\renewcommand{\labelenumi}{(\alph{enumi})}
\renewcommand\theenumi\labelenumi

\newtheorem{theorem}{Theorem}
\newtheorem{lemma}[theorem]{Lemma}

\newcommand{\ooea}{(1+1)~EA\xspace}

\newcommand{\oplea}{(1+$\lambda$)~EA\xspace}
\newcommand{\mpoea}{($\mu$+1)~EA\xspace}
\newcommand{\oclea}{(1,$\lambda$)~EA\xspace}

\newcommand{\om}{\textsc{OneMax}\xspace}
\newcommand{\onemax}{\om}
\newcommand{\lo}{\textsc{LeadingOnes}\xspace}
\newcommand{\leadingones}{\lo}
\newcommand{\R}{\ensuremath{\mathbb{R}}}
\newcommand{\N}{\ensuremath{\mathbb{N}}}
\newcommand{\Z}{\ensuremath{\mathbb{Z}}}
\DeclareMathOperator{\Prob}{Pr}
\DeclareMathOperator{\init}{init}
\DeclareMathOperator{\Bin}{Bin}
\newcommand{\xmin}{x_{\mathrm{min}}}
\newcommand{\xmax}{x_{\mathrm{max}}}

\newcommand{\ie}{i.\,e.\xspace}

\newcommand{\Var}{\mathrm{Var}\xspace} 
\newcommand{\eps}{\varepsilon} 

\newenvironment{proofof}[1]{\begin{proof}[Proof of~#1]}{\end{proof}}

\title{The (1+$\lambda$)~Evolutionary Algorithm with Self-Adjusting Mutation Rate\thanks{An extended abstract 
of this report appeared in the proceedings of the 2017 Genetic and Evolutionary Computation Conference (GECCO~2017) \cite{DoerrGWYGECCO17}.}}

\author{Benjamin Doerr\\
Laboratoire d'Informatique (LIX)\\{\'E}cole Polytechnique\\Palaiseau, France
\and
Christian Gießen\\
DTU Compute\\Technical University of Denmark\\Kgs.\ Lyngby, Denmark
\and
Carsten Witt\\
DTU Compute\\Technical University of Denmark\\Kgs.\ Lyngby, Denmark
\and
Jing Yang\\
Laboratoire d'Informatique (LIX)\\{\'E}cole Polytechnique\\Palaiseau, France}

\begin{document}
\maketitle

\begin{abstract}
We propose a new way to self-adjust the mutation rate in population-based evolutionary algorithms in discrete search spaces. Roughly speaking, it consists of creating half the offspring with a mutation rate that is twice the current mutation rate and the other half with half the current rate. The mutation rate is then updated to the rate used in that subpopulation which contains the best offspring. 

We analyze how the $(1+\lambda)$ evolutionary algorithm with this self-adjusting mutation rate optimizes the OneMax test function.  
We prove that this dynamic version of the $(1+\lambda)$~EA finds the optimum in an expected optimization time (number of fitness evaluations) of $O(n\lambda/\log\lambda+n\log n)$. This time is asymptotically smaller than the optimization time of the classic $(1+\lambda)$ EA. Previous work shows that this performance is best-possible among all $\lambda$-parallel mutation-based unbiased black-box algorithms.

This result shows that the new way of adjusting the mutation rate can find optimal dynamic parameter values on the fly. Since our adjustment mechanism is simpler than the ones previously used for adjusting the mutation rate and does not have parameters itself, we are optimistic that it will find other applications.
\end{abstract}

\section{Introduction}

Evolutionary algorithms (EAs) have shown a remarkable performance in a broad range of applications.
However, it has often been observed that this performance depends crucially on the use of the right parameter settings. Parameter optimization and parameter control are therefore key topics in EA research. Since these have very different characteristics in discrete and continuous search spaces, we discuss in this work only evolutionary algorithms for discrete search spaces.

Theoretical research 
has contributed to our understanding of these algorithms with mathematically founded runtime analyses, many of which show how the runtime of an EA is determined by its parameters.
The majority of these works investigate \emph{static parameter settings},
\ie, the parameters are fixed before the start of the algorithm and are not changed during its execution. More recently, a number of results were shown which prove an advantage of \emph{dynamic parameter settings}, that is, the parameters of the algorithm are changed during its execution. Many of these rely on making the parameters functionally dependent on the current state of the search process, e.g., on the fitness of the current-best individual. While this provably can lead to better performances, it leaves the algorithm designer with an even greater parameter setting task, namely inventing a suitable functional dependence instead of fixing numerical values for the parameters. This problem has been solved by theoretical means for a small number of easy benchmark problems, but it is highly unclear how to find such functional relations in the general case.

A more designer-friendly way to work with dynamic parameters is to modify the parameters based on simple rules taking into account the recent performance. A number of recent results shows that such \emph{on the fly} or \emph{self-adjusting} parameter settings can give an equally good performance as the optimal fitness-dependent parameter setting, however, with much less input from the algorithm designer. For example, good results have been obtained by increasing or decreasing a parameter depending on whether the current iteration improved the best-so-far solution or not, e.g., in a way resembling the $1/5$-th rule from continuous optimization.

Such success-based self-adjusting parameter settings can work well when there is a simple monotonic relation between success and parameter value, e.g., when one speculates that increasing the size of the population in an EA helps when no progress was made. For parameters like the mutation rate, it is not clear what a success-based rule can look like, since a low success rate can either stem from a too small mutation rate (regenerating the parent with high probability) or a destructive too high mutation rate. In~\cite{DoerrYangGECCO16}, a relatively complicated learning mechanism was presented that tries to learn the right mutation strength by computing a time-discounted average of the past performance stemming from different parameter values. This learning mechanism needed a careful trade-off between exploiting the currently most profitably mutation strength and experimenting with other parameter values and a careful choice of the parameter controlling by how much older experience is taken less into account than more recent observations.

\subsection{A New Self-Adjusting Mechanism for Population-Based EAs}

In this work, we propose an alternative way to adjust the mutation rate on the fly for algorithms using larger offspring populations. It aims at overcoming some of the difficulties of the learning mechanism just described. The simple idea is to create half the offspring with twice the current mutation rate and the other half using half the current rate. The mutation rate is then modified to the rate which was used to create the best of these offspring (choosing the winning offspring randomly among all best in case of ambiguity). We do not allow the mutation rate to leave the interval $[2/n, 1/4]$, so that the rates used in the subpopulations are always in the interval $[1/n, 1/2]$.

We add one modification to the very basic idea described in the first paragraph of this section. Instead of always modifying the mutation rate to the rate of the best offspring, we shall take this winner's rate only with probability a half and else modify the mutation rate to a random one of the two possible values (twice and half the current rate). Our motivation for this modification is that we feel that the additional random noise will not prevent the algorithm from adjusting the mutation rate into a direction that is more profitable. However, the increased amount of randomness may allow the algorithm to leave a possible basin of attraction of a locally optimal mutation rate. Observe that with probability $\Theta(1/n^2)$, a sequence of $\log_2 n$ random modifications all in the same direction appears. Hence with this inverse-polynomial rate, the algorithm can jump from any mutation rate to any other (with the restriction that only a discrete set of mutation rates can appear). We note that the existence of random modifications is also exploited in our runtime analysis, which will show that the new self-adjusting mechanism selects mutation rates good enough to lead to the asymptotically optimal runtime among all dynamic choices of the mutation rate for the \oplea.

In this first work proposing this mechanism, we shall not spend much effort fine-tuning it, but rather show in a proof-of-concept manner that it can find very good mutation rates. In a real application, it is likely that better results are obtained by working with three subpopulations, namely an additional one using (that is, exploiting) the current mutation rate. Also, it seems natural that more modest adjustments of the mutation rate, that is, multiplying and dividing the rate by a number $F$ that is smaller than the value $F=2$ used by our mechanism, is profitable. We conduct some elementary experiments supporting this intuition in Section~\ref{sec:experiments}.

%
%

\subsection{Runtime Analysis for the Self-Adjusting \oplea on \onemax}

To prove that the self-adjusting mechanism just presented can indeed find good dynamic mutation rates, we conduct a rigorous runtime analysis for our self-adjusting \oplea on the classic test function \[\onemax : \{0,1\}^n \to \R; (x_1,\dots,x_n) \mapsto \sum_{i=1}^n x_i.\]

The runtime of the \oplea with fixed mutation rates on \om is well understood~\cite{DoerrKuennemannTCS15,GiessenWittALGO16}.
In particular, \citet{GiessenWittALGO16} show that the expected runtime (number of generations) 
is $(1\pm o(1))\left(\frac{1}{2}\cdot\frac{n\ln\ln\lambda}{\ln\lambda}+\frac{e^r}{r}\cdot\frac{n\ln n}{\lambda}\right)$
when a mutation rate of $r/n$, $r$~a constant, is used. Thus for $\lambda$ not too large,
the mutation rate determines the leading constant of the runtime, and a rate of $1/n$ gives the asymptotically best runtime.

As a consequence of their work on parallel black-box complexities, Badkobeh, Lehre, and Sudholt~\cite{BadkobehPPSN14} showed that the \oplea with a suitable fitness-dependent mutation rate finds the optimum of \om in an asymptotically better runtime of $O(\frac{n}{\log\lambda}+\frac{n\log n}{\lambda})$, where 
the improvement is by a factor of $\Theta(\log\log\lambda)$.\footnote{As usual in the analysis of algorithms, we write $\log$ when there is no need to specify the base of the logarithm, e.g., in asymptotic terms like $\Theta(\log n)$. To avoid the possible risk of an ambiguity for small arguments, as usual, we set $\log x = \max\{1, \log x\}$. All other logarithms are used in their classic meaning, that is, for all $x > 0$, we denote by $\ln x$ the natural logarithm of $x$ and by $\log_2 x$ its binary logarithm.} This runtime is best-possible among all $\lambda$-parallel unary unbiased black-box optimization algorithms. In particular, no other dynamic choice of the mutation rate in the \oplea can achieve an asymptotically better runtime. The way how the mutation rate depends on the fitness in the above result, however, is not trivial. When the parent individual has fitness distance $d$, then mutation rate employed is $p = \max\{\frac{\ln \lambda}{n \ln(en/d)}, \frac 1n\}$. 

%


Our main technical result is that the \oplea adjusting the mutation rate according to the mechanism described above has the same (optimal) asymptotic runtime. Consequently, the self-adjusting mechanism is able to find on the fly a mutation rate that is sufficiently close to the one proposed in~\cite{BadkobehPPSN14} to achieve asymptotically the same expected runtime.

\begin{theorem}
\label{thm:main}
Let $\lambda \ge 45$ and $\lambda=n^{O(1)}$. 
Let $T$ denote the number of generations 
of the \oplea with self-adjusting mutation rate on \om.
Then,
\begin{equation*}
E(T) = \Theta\left( \frac{n}{\log \lambda} + \frac{n\log n}{\lambda}  \right).
\end{equation*}
This corresponds to an expected number of function evaluations of 
$\Theta(\frac{\lambda n}{\log\lambda}+n\log n)$.
\end{theorem}

To the best of our knowledge,
this is the first time that a simple mutation-based EA
achieves a super-constant speed-up via a self-adjusting choice of the mutation rate.

As an interesting side remark, our proofs reveal that 
a quite non-standard but fixed mutation rate of $r=\ln(\lambda)/2$ also achieves the $\Theta(\log\log\lambda)$ improvement as it 
implies the bound of $\Theta(n/\log \lambda)$ 
generations if $\lambda$ is not too small. Hence, the constant choice $r=O(1)$ as studied in 
\cite{GiessenWittALGO16} does not yield the asymptotically optimal number of generations unless 
$\lambda$ is so small that the $n\log n$-term dominates.

\begin{lemma}
\label{lem:main-2}
Let $\lambda\ge 45$ and $\lambda=n^{O(1)}$, 
Let $T$ denote the number of generations 
of the \oplea with fixed mutation rate $r=\ln(\lambda)/2$.
Then,
\begin{equation*}
E(T) = O\left( \frac{n}{\log \lambda} + \frac{n\log n}{\sqrt{\lambda}}  \right).
\end{equation*}
This corresponds to an expected number of function evaluations of 
$O(\frac{\lambda n}{\log\lambda}+\sqrt{\lambda} n\log n)$.
\end{lemma}

The paper is structured as follows:
In Section~\ref{sec:related} we  give an overview
over previous analyses of the \oplea and of self-adjusting parameter control mechanism
in EAs from a theoretical perspective.
In Section~\ref{sec:preliminaries} we give the algorithm and the mutation scheme.
For convenience, we also state some key theorems
that we will frequently use in the rest of the paper. Afterwards, Section~\ref{sec:proof-overview} 
gives a high-level overview of our main proof strategy. 
The following technical sections deal with the runtime analysis
of the expected time spent by the \oplea on \om
in each of three regions of the fitness distance~$d$.
We label these regions the \emph{far region}, \emph{middle region} and \emph{near region}, each 
of which 
will be dealt with in a separate section. The proof of the main theorem and of Lemma~\ref{lem:main-2} is then given in 
Section~\ref{sec:main-theo}. 
%
Finally, we conclude in Section~\ref{sec:conclusion}.

\section{Related Work}
\label{sec:related}

Since this is a theoretically oriented work on how a dynamic parameter choice speeds up the runtime of the \oplea on the test function \onemax, let us brief{}ly review what is known about the theory of this EA and dynamic parameter choices in general. 

\subsection{The \oplea}

The first to conduct a rigorous runtime analysis of the \oplea were Jansen, De Jong, and Wegener~\cite{JansenDW05}. They proved, among other results, that when optimizing  \onemax  a linear speed-up exists up to a population size of $\Theta(\log(n)\log\log(n)/\log\log\log(n))$, that is, for $\lambda = O(\log(n)\log\log(n)/\log\log\log(n))$, finding the optimal solution takes an expected number of $\Theta(n \log(n) / \lambda)$ generations, whereas for larger $\lambda$ at least $\omega(n \log(n)/\lambda)$ generations are necessary. This picture was completed in~\cite{DoerrKuennemannTCS15} with a proof that the expected number of generations taken to find the optimum is $\Theta(\frac{n \log n}{\lambda} + \frac{n \log\log\lambda}{\log\lambda})$. The implicit constants were determined in~\cite{GiessenWittALGO16}, 
giving the bound of $(1 \pm o(1)) (\frac 12 \frac{n \ln\ln\lambda}{\ln\lambda}+\frac{e^r}{r} \frac{n\ln n}{\lambda})  $, for any constant~$r$, as 
mentioned in the introduction.

Aside from the optimization behavior on \onemax, not too much is known for the \oplea, or is at least not made explicit (it is easy to see that waiting times for an improvement which are larger than $\lambda$ reduce by a factor of $\Theta(\lambda)$ compared to one-individual offspring populations). Results made explicit are the $\Theta(n^2 / \lambda + n)$ expected runtime (number of generations) on \leadingones~\cite{JansenDW05}, the worst-case $\Theta(n+ n \log(n) / \lambda )$ expected runtime on linear functions~\cite{DoerrKuennemannTCS15}, and the $O(m^2 (\log n + \log w_{\max}) / \lambda)$ runtime estimate for minimum spanning trees valid for $\lambda \le m^2/n$~\cite{NeumannW07}, where $n$ denotes the number of vertices of the input graph, $m$ the number of edges, and $w_{\max}$ the maximum of the integral edge weights.

\subsection{Dynamic Parameter Choices}

While it is clear the EAs with parameters changing during the run of the algorithm (dynamic parameter settings) can be more powerful than those only using static parameter settings, only recently considerable advantages of dynamic choices could be demonstrated by mathematical means (for discrete optimization problems; in continuous optimization, step size adaptation is obviously necessary to approach arbitrarily closely a target point). To describe the different ways to dynamically control parameters, we use in the following the language proposed in Eiben, Hinterding, and Michalewicz~\cite{EibenHM99} and its extension from~\cite{DoerrD15}. 

\subsubsection{Deterministic Parameter Control}

In this language, \emph{deterministic parameter control} means that the dynamic choice of a parameter does not depend on the fitness landscape. The first to rigorously analyze a deterministic parameter control scheme are Jansen and Wegener~\cite{JansenW06}. They regard the performance of the \ooea which uses in iteration $t$ the mutation rate $k/n$, where $k \in \{1, 2, \ldots, 2^{\lceil \log_2 n \rceil-2}\}$ is chosen such that $\log_2(k) \equiv t-1 \,\,\,(\text{mod}\, \lceil \log_2 n \rceil-1)$. In other words, they cyclically use the mutation rates $1/n, 2/n, 4/n, \ldots, K/n$, where $K$ is the largest power of two that is less than $n/2$. Jansen and Wegener demonstrate that there exists an example function where this dynamic EA significantly outperforms the \ooea with any static mutation rate. However, they also observe that for many classic problems, this EA is slower by a factor of $\Theta(\log n)$. 


In~\cite{OlivetoCEC09}, a rank-based mutation rate was analyzed for the \mpoea. 
A previous experimental study~\cite{CervantesS08} suggested that this is a profitable approach, but the mathematical runtime analysis in~\cite{OlivetoCEC09} rather indicates the opposite. While there are artificial examples where a huge runtime gain could be shown and also the worst-case runtime of the \mpoea reduces from essentially $n^n$ to $O(3^n)$, a rigorous analysis on the \onemax function rather suggests that the high rate of offspring generated with a mutation rate much higher than $1/n$ brings a significant risk of slowing down the optimization process.

For two non-standard settings in evolutionary computation, deterministic parameter control mechanisms also gave interesting results. For problems where the solution length is not known~\cite{CathabardLY11}, more precisely, where the number or the set of bits relevant for the solution quality is unknown, again random mutation rates gave good results~\cite{DoerrDK15,DoerrDK17}. Here however, not a power-law scheme was used, but rather one based on very slowly decreasing summable sequences. For problems where the discrete variables take many values, e.g., the search space is $\{0, \dots, r-1\}^n$ for some large $r$, the question is how to change the value of an individual variable. The results in~\cite{DoerrDK16} suggest that a harmonic mutation strength, that is, changing a variable value by $\pm i$ with $i$ chosen randomly with probability proportional to $1/i$, can be beneficial. This distribution was analyzed earlier in~\cite{DietzfelbingerRWW10} for the one-dimensional case, where it was also shown to give the asymptotically best performance on a \onemax type problem.

For randomized search heuristics outside evolutionary computation, Wegener~\cite{Wegener05} showed that simulated annealing (using a time-dependent temperature) can beat the Metropolis algorithm (using a static temperature).

\subsubsection{Adaptive Parameter Control}

A parameter control scheme is called \emph{adaptive} if it used some kind of feedback from the optimization process. This can be functionally dependent (e.g., the mutation rate depends on the fitness of the parent) or success-based (e.g., a 1/5th rule). 

The first to conduct a runtime analysis for an adaptive parameter control mechanism (and show a small advantage over static choices) were B\"ottcher, Doerr, and Neumann~\cite{BottcherDN10}. They proposed to use the \emph{fitness-dependent} mutation rate of $1/(\lo(x)+1)$ for the optimization of the \leadingones test function. They proved that with this choice, the runtime of the \ooea improves to roughly $0.68n^2$ compared to a time of $0.86n^2$ stemming from the classic mutation rate $1/n$ or a runtime of $0.77n^2$ stemming from the asymptotically optimal static rate of approximately $1.59/n$. 

\begin{sloppypar}
For the $(1+(\lambda,\lambda))$~GA, a fitness-dependent offspring population size of order $\lambda = \Theta(\sqrt{n/d(x)}\,)$ was suggested in~\cite{DoerrEbelTCS15}, where $d(x)$ is the fitness-distance of the parent individual to the optimum. This choice improves the optimization time (number of fitness evaluations until the optimum is found) on \onemax from $\Theta(n\sqrt{\log(n)\log\log\log(n)/\log\log (n)}\,)$ stemming from the optimal static parameter choice~\cite{Doerr16} to $O(n)$. Since in this adaptive algorithm the mutation rate $p$ is functionally dependent on the offspring population size, namely via $p = \lambda/n$, the dynamic choice of $\lambda$ is equivalent to a fitness-dependent mutation rate of $1/\sqrt{nd(x)}$. 
\end{sloppypar}

In the aforementioned work by Badkobeh et al.~\cite{BadkobehPPSN14}, a fitness-dependent mutation rate of $\max\bigl\{\frac{\ln\lambda}{n \ln(en/d(x))}, \frac 1n\bigr\}$ was shown to improve the classic runtime of $O\bigl(\frac{n \log\log\lambda}{\log\lambda}+\frac{n \log n}{\lambda} \bigr)$ to $O\bigl(\frac{n}{\log\lambda}+\frac{n \log n}{\lambda} \bigr)$. In~\cite{DoerrYangGECCO16}, the \ooea using a $k$-bit flip mutation operator together with a fitness-dependent choice of $k$ was shown to give a performance on \onemax that is very close to the theoretical optimum (among all unary unbiased black-box algorithms), however, this differs only by lower order terms from the performance of the simple randomized local search heuristic (RLS). For nature-inspired algorithms other than evolutionary ones, Zarges~\cite{Zarges08,Zarges09} proved that fitness-dependent mutation rates can be beneficial in artificial immune systems.

\subsubsection{Self-adjusting and Self-adaptive Parameter Control}

While all these results show an advantage of adaptive parameter settings, it remains questionable if an algorithm user would be able to find such a functional dependence of the parameter on the fitness. This difficulty can be overcome via \emph{self-adjusting} parameter choices, where the parameter is modified according to a simple rule often based on the success or progress of previous iterations, or via \emph{self-adaptation}, where the parameter is encoded in the genome and thus subject to variation and selection. The understanding of self-adaptation is still very limited. The only theoretical work on this topic~\cite{DangL16}, however, is promising and shows examples where self-adaptation can lead to significant speed-ups for non-elitist evolutionary algorithms.\footnote{Note added in proof: Very recently, it was shown in~\cite{DoerrWY18} that the \oclea with self-adaptive mutation rate has a runtime asymptotically equivalent to the one of the algorithm regarded in this work.}

In contrast to this, the last years have produced a profound understanding of self-adjusting parameter choices. The first to perform a mathematical analysis were \citet{LaessigFOGA11}, who considered the \oplea and a simple parallel island model together with two self-adjusting mechanisms for population size or island number, including halving or doubling it depending on whether the current iteration led to an improvement or not. These mechanisms were proven to give significant improvements of the ``parallel'' runtime (number of generations) on various test functions  without increasing significantly the ``sequential'' runtime (number of fitness evaluations).

In~\cite{DoerrD15} it was shown that the fitness-dependent choice of $\lambda$ for the $(1+(\lambda,\lambda))$~GA described above can also be found in a self-adjusting way. To this aim, another success-based mechanism was proposed, which imitates the $1/5$-th rule from evolution strategies. With some modifications, this mechanism also works on random satisfiability problems~\cite{BuzdalovD17}. For the problem of optimizing an $r$-valued \onemax function, a self-adjustment of the step size inspired by the $1/5$-th rule was found to find the asymptotically best possible runtime in~\cite{DoerrKoetzingPPSN16}. 

These results indicate that success-based dynamics work well for adjusting parameters when a monotonic relation like ``if progress is difficult, then increase the population size'' holds. For adjusting a parameter like the mutation rate, it is less obvious how to do this. For example, in the search space $\{0,1\}^n$ both a too large mutation rate (creating a stronger drift towards a Hamming distance of $n/2$ from the optimum)  and a too small mutation rate (giving a too small radius of exploration) can be detrimental. For this reason, to obtain a self-adjusting version of their result on the optimal number $k$ to optimize \onemax via $k$-bit flips~\cite{DoerrYangGECCO16}, in~\cite{DoerrYangPPSN16} a learning mechanism was proposed that from the medium-term past estimates the efficiency of different parameter values. As shown there, this does find the optimal mutation strength sufficiently well to obtain essentially the runtime stemming from the fitness-dependent mutation strength exhibited before. 

In the light of these works, our result from the methodological perspective shows that some of the difficulties of the learning mechanism of~\cite{DoerrYangPPSN16}, e.g., the whole book-keeping being part of it and also the setting of the parameters regulating how to discount information over time, can be overcome by the mechanism proposed in this work. In a sense, the use of larger populations enables us to adjust the mutation rate solely on information learned in the current iteration. However, we do also use the idea of~\cite{DoerrYangPPSN16} to intentionally use parameter settings which appear to be slightly off the current optimum to gain additional insight.

\subsubsection{Selection Hyper-Heuristics}

We note that so-called \emph{selection hyper-heuristics} may lead to processes resembling dynamic parameter choices. Selection hyper-heuristics are methods that select, during the run of the algorithm, which one out of several pre-specified simpler algorithmic building blocks to use. When the different pre-specified choices are essentially identical apart from an internal parameter, then this selection hyper-heuristic approach could also be interpreted as a dynamic choice of the internal parameter. For example, when only the two mutation operators are available that flip exactly one or exactly two bits, then a selection hyper-heuristic choosing between them could also be interpreted as the \emph{randomized local search} heuristic using a dynamic choice of the number of bits it flips. We brief{}ly review the main runtime analysis works of this flavor.

The first to conduct a rigorous runtime analysis for selection hyper-heuristics were Lehre and \"Oczan~\cite{LehreO13}. They show that the \ooea using the \emph{mixed strategy} of choosing the mutation operator randomly between the 1-bit flip operator (with probability $p$) and the $2$-bit flip operator (with probability $1-p$) optimizes the \onemax function in time 
$\min\{\frac 1p \, n (1+\ln n), \,\frac{1}{1-p} \, n^2 (1-\frac 1n)\}$. 
It appears to us that, most likely, this result is not absolutely correct, since, e.g., in the case $p=0$ the expected optimization time is clearly infinite (if the random initial search point has an odd Hamming distance from the optimum, then the optimum cannot be reached only via $2$-bit flips). So most likely, the upper bound should be increased by an additive term of $\frac 1p n$. 

Lehre and \"Oczan~\cite{LehreO13} further construct an example function \textsc{GapPath}, which has the property that it can only be optimized when both the $1$-bit and the $2$-bit flip mutation operator occur with positive probability and use this to discuss various static and dynamic ways to randomly decide between the two operators. A similar example for a multi-objective optimization setting was given by Qian, Tang, and Zhou~\cite{QianTZ16}.

Note that for $p=\frac 12$, the randomized selection heuristic is the classic \emph{randomized search heuristic} using $1$-bit and $2$-bit flips, which was regarded, among others, in~\cite{GielW03,NeumannW07}. 

In~\cite{AlanaziL14}, Alanazi and Lehre extent the previous work to several classic selection hyper-heuristics like \emph{simple random} (take a random low-level heuristic in each iteration), \emph{random gradient} (take a random low-level heuristic and repeat using it as long as a true fitness improvement is obtained), \emph{greedy} (in each iteration, use all low-level heuristic in parallel and continue with best result obtained), \emph{permutation} (generate initially a random cyclic order of the low-level heuristics and then use them in this order). Again for the choice between $1$-bit and $2$-bit flips, they prove upper and lower bounds for the expected optimization time on the \leadingones benchmark function. While the results are relatively tight (the corresponding upper and lower bounds deviate by at most a factor of $6$), the intervals of possible asymptotic runtimes intersect, so this first runtime analysis of this type does not yet give a conclusive picture on how the different selection heuristics compare.

Given that the probabilities to find a true improvement are very low in this discrete optimization problem, one would expect that all these four heuristics have the same runtime (apart from lower order terms), and this is indeed the first set of results by Lissovoi, Oliveto, and Warwicker~\cite{LissovoiOW17}, who show that the expected runtime in all cases is $(1+o(1)) \frac 12 \ln(3) n^2$. They build on this strong result by proposing to use a slower adaptation. For the random gradient method, they propose to use a random low-level heuristic for up to $\tau$ iterations. If an improvement is found, immediately another phase with this operator is started. If a phase of $\tau$ iterations does not see a fitness improvement, then a new phase is started with a random operator. 

For this \emph{generalized random gradient} heuristic using a phase length of $\tau=cn$ for a constant $c$, they show (still for the \leadingones problem and the $1$-bit and $2$-bit mutation operator) an expected runtime of $(1+o(1)) g(c) n^2$, where $g(c)$ is a constant depending on $c$ only that tends to $\frac{\ln(2)+1}{4}$ when $c$ is tending to infinity. Consequently, this new hyper-heuristic outperforms the previously investigated ones and approaches arbitrarily well the best-possible performance that can be obtained from the two mutation operators, which is, as also shown in~\cite{LissovoiOW17}, $(1+o(1)) \frac{\ln(2)+1}{4} n^2$.

The \emph{generalized random gradient} heuristic was further extended in~\cite{DoerrLOW18}. There an operator was defined as successful (which leads to another phase using this operator) if it leads to $\sigma$ improvements in a phase of at most $\tau$ iterations. Hence in this language, the previous \emph{generalized random gradient} heuristic uses $\sigma=1$. By using a larger value of $\sigma$, the algorithm is able to take more robust decisions on what is a success. This was used in~\cite{DoerrLOW18} to determine the phase length $\tau$ in a self-adjusting manner. While the previous work~\cite{LissovoiOW17} does not state this explicitly, the choice of $\tau$ is crucial. A $\tau$-value of smaller asymptotic order than $\Theta(n)$ leads to typically no improvement within a phase and thus reverts the algorithms to the simple random heuristic. A $\tau$-value of more than $c n \ln(n)$, where $c$ is a suitable constant, results in that both operators are successful in most parts of the search space. Consequently, the algorithm sticks to the first choice for a large majority of the optimization process and thus does not profit from the availability of both operators. 

Since the choice of $\tau$ is that critical, a mechanism successfully adjusting it to the right value is  desirable. In~\cite{DoerrLOW18} it is shown that by choosing $\sigma \in \Omega(\log^4 n) \cap o(\sqrt{n / \log n})$---note that this is a quite wide range---the value of $\tau$ can be easily adjusted on the fly via a multiplicative update rule resembling vaguely the $1/5$-th success rule. This gives again the asymptotically optimal runtime of  $(1+o(1)) \frac{\ln(2)+1}{4} n^2$.

\section{Preliminaries}
\label{sec:preliminaries}
We shall now formally define the algorithm analyzed and present some fundamental 
tools for the analysis.

\subsection{Algorithm}
We consider the \oplea with self-adjusting mutation rate for the minimization
of pseudo-boolean functions $f:\{0,1\}^n\to\R$,
defined as Algorithm~\ref{alg:onelambda}.

The general idea of the mutation scheme is to adjust the mutation strength
according to its success in the population.
We perform mutation by applying standard bit mutation
with two different mutation probabilities $r/(2n)$ and $2r/n$
and we call $r$ the \emph{mutation rate}.
More precisely, for an even number $\lambda\geq 2$
the algorithm creates $\lambda/2$ offspring with mutation rate $r/2$ and with $2r$ each.

The mutation rate is adjusted after each selection. With probability a half, the new rate is taken as the mutation rate
that the best individual (\ie the one with the lowest fitness, ties broken uniformly at random) was created with (\emph{success-based adjustment}). With the other 50\% probability, the mutation rate is adjusted to a random value in $\{r/2, 2r\}$ (\emph{random adjustment}). 
Note that the mutation rate is adjusted in each iteration, that is, also when all offspring are worse than the parent and thus the parent is kept for the next iteration. 

If an adjustment of the rate results in a new rate $r$ outside the interval $[2,n/4]$, we replace this rate with the corresponding boundary value. Note that in the case of $r<2$, a subpopulation with rate less than $1$ would be generated, which means flipping less than one bit in expectation.
At a rate $r > n/4$, a subpopulation with rate larger than $n/2$ would be created, which again is not a very useful choice.

We formulate the algorithm to start with an initial mutation rate $r^{\init}$.
The only assumption on $r^{\init}$ is to be greater than or equal to $2$. 
The \oplea with this self-adjusting choice of the mutation rate is given as pseudocode in Algorithm~\ref{alg:onelambda}.

\begin{algorithm}
\caption{\oplea with two-rate standard bit mutation}
\label{alg:onelambda}
	\begin{algorithmic}
		\State Select $x$ uniformly at random from $\{0, 1\}^n$ and set $r \gets r^{\init}$.
		\For{$t \gets 1, 2, \dots$}
			\For{$i \gets 1, \dots, \lambda$}
        \State Create $x_i$ by flipping each bit in a copy of $x$ independently with probability $r_t/(2n)$ if $i\leq\lambda/2$ and with probability $2r_t/n$ otherwise.
			\EndFor
			\State $x^* \gets \arg\min_{x_i} f(x_i)$ (breaking ties randomly).
			\If{$f(x^*) \leq f(x)$}
				\State $x \gets x^*$.
			\EndIf
      \State Perform one of the following two actions with prob.~$1/2$:
      \begin{itemize}
      {\setlength\itemindent{10pt}}\item Replace $r_t$ with the mutation rate that $x^*$ has been created with.
      {\setlength\itemindent{10pt}}\item Replace $r_t$ with either $r_t/2$ or $2r_t$, each with probability~$1/2$.
      \end{itemize}
      \State Replace $r_t$ with $\min\{\max\{2, r_t\}, n/4\}$.
		\EndFor
	\end{algorithmic}
\end{algorithm}


Let us explain the motivation for the random adjustments of the rate. Without such random adjustments, the rate can only be changed into some direction if a winning offspring is generated with this rate. For simple functions like \onemax, this is most likely sufficient. However, when the fitness of the best of $\lambda/2$ offspring, viewed as a function of the rate, is not unimodal, then several adjustments into a direction at first not yielding good offspring might be needed to reach good values of the rate. Here, our random adjustments enable the algorithm to cross such a valley of unfavorable rate values. We note that such ideas are not uncommon in evolutionary computation, with standard-bit mutation being the most prominent example (allowing to perform several local-search steps in one iteration to cross fitness valleys). 

A different way to implement a mechanism allowing larger changes of the rate to cross unfavorable regions would have been to not only generate offspring with rates $r/2$ and $2r$, but to allow larger deviations from the current rate with some small probability. One idea could be choosing for each offspring independently the rate $r 2^{-i}$ with probability $2^{-|i|-1}$ for all $i \in \Z$, $i \neq 0$. This should give similar results, but to us the process appears more chaotic (e.g., with not the same number of individuals produced with rates $r/2$ and $2r$). 


The \emph{runtime}, also called the \emph{optimization time}, of the \oplea
is the smallest~$t$ such that an individual of minimum $f$\nobreakdash-value
has been found.
Note that $t$ corresponds to a number of iterations (also called generations),
where each generation creates $\lambda$ offspring.
Since each of these offspring has to be evaluated,
the number of function evaluations, which is a classical cost measure,
is by a factor of $\lambda$ larger than the runtime as defined here.
However, assuming a massively parallel architecture
that allows for parallel evaluation of the offspring,
counting the number of generations seems also a valid cost measure.
In particular, a speed-up on the function $\om(x_1,\dots,x_n):=x_1+\dots+x_n$ by increasing~$\lambda$ can only
be observed in terms of the number of generations.
Note that for reasons of symmetry, it makes no difference
whether \om is minimized (as in the present paper)
or maximized (as in several previous research papers). 

Throughout the paper, all asymptotic notation
will be with respect to the problem size~$n$. 

\subsection{Drift Theorems}
Our results are obtained by drift analysis,
which is also used in previous 
analyses of the \oplea without self-adaptation 
on \om and other linear functions \cite{DoerrKuennemannTCS15,GiessenWittALGO16}.

The first theorems stating upper bounds on the hitting time
using variable drift go back to \cite{Johannsen10,MitavskiyVariable}.
We take a formulation from \cite{LehreWittISAAC14} but simplify it 
to Markov processes for notational convenience.

\begin{theorem}[Variable Drift, Upper Bound]
\label{theo:variable-upper}
Let $(X_t)_{t\ge 0}$, be random variables
describing a Markov process over a finite state space
 $S\subseteq \{0\}\cup [\xmin,\xmax]$, where $\xmin>0$.
Let $T$ be the random variable that denotes the earliest point in time $t \geq 0$
such that $X_t = 0$. 
If there exists a monotone increasing function $h(x)\colon [\xmin,\xmax]\to\R^+$, where   
$1/h(x)$ is integrable on $[\xmin,\xmax]$, such that  for all $x\in S$ with 
$\Prob(X_t=x)>0$  we have 
\[
E(X_t-X_{t+1} \mid X_t=x) \ge h(x)
\]
 then for all $x'\in S$ with $\Prob(X_0=x')>0$
\begin{equation*}
  E(T\mid X_0=x') \le
\frac{\xmin}{h(\xmin)} + \int_{\xmin}^{x'} \frac{1}{h(x)} \,\mathrm{d}x.
\end{equation*}
\end{theorem}

The variable drift theorem is often applied in the special case of \emph{additive drift} in 
discrete spaces: 
assuming $E(X_t-X_{t+1} \mid X_t=x; X_t>0) \ge \eps$ for some constant~$\eps$,  
one obtains $E(T\mid X_0=x')\le x'/\eps$.

Since we will make frequent use of it in the following sections as well,
we will also give the version of the \emph{Multiplicative Drift Theorem} for upper bounds,
due to \cite{DoerrJohannsenWinzenALGO12}. Again, 
 this is  implied by the previous variable drift theorem.

\begin{theorem}[Multiplicative Drift~\cite{DoerrJohannsenWinzenALGO12}]
\label{thm:multiDrift}
Let $(X_t)_{t\geq 0}$ be random variables
describing a Markov process over a finite state space $S \subseteq \mathbb{R}^+_0$
and let $\xmin := \min\{x \in S \; | \; x > 0\}$.
Let $T$ be the random variable that denotes the earliest point in time $t \geq 0$
such that $X_t = 0$.
If there exist $\delta > 0$ such that for all $x\in S$ with $\Prob(X_t=x)>0$ we have
\begin{equation*}
  E(X_t-X_{t+1} \mid X_t = x) \geq \delta x\enspace,
\end{equation*}
then for all $x'\in S$ with $\Prob(X_0=x')>0$,
\begin{equation*}
  E(T \mid X_0=x')\leq \frac{1+\ln\left(\frac{x'}{\xmin}\right)}{\delta}\enspace.
\end{equation*}
\end{theorem}

\subsection{Chernoff Bounds}


%
%
For reasons of self-containedness,
we state two well-known multiplicative Chernoff bounds
and a lesser known additive Chernoff bound
that is also known in the literature as Bernstein's inequality. These bounds can be found, e.g., in~\cite{AugerDoerr11}.

\begin{theorem}[Bernstein's inequality, Chernoff bounds]\label{thm:chernoff}
	Let $X_1, \ldots, X_n$ be independent random variables and $X = \sum_{i = 1}^n X_i$. 
	\begin{enumerate}
	\item \label{bernstein}Let $b$ be such that $E(X_i) - b \le X_i \le E(X_i)+b$ for all $i = 1, \ldots, n$. Let $\sigma^2 = \sum_{i=1}^n \Var(X_i) = \Var[X]$. Then for all $\Delta \ge 0$,
	\begin{align*}
	\Pr(X \ge E(X) + \Delta) &\le 	\exp\bigg(-\frac{\Delta^2}{2(\sigma^2+\frac 13 b \Delta)}\bigg).
	\end{align*}
  \item Assume that for all $i=1,\dots,n$, the random variable $X_i$ takes values in $[0,1]$ only. Then
	\begin{itemize}
		\item $\Prob(X\leq (1-\delta)E(X))\leq\exp(-\delta^2E(X)/2)$ for all $0\le \delta\le 1$,
		\item $\Prob(X\geq (1+\delta)E(X))\leq\exp(-\delta^2E(X)/(2+\delta))$ for all $\delta>0$.
	\end{itemize}
  \end{enumerate}
\end{theorem}

\subsection{Occupation Probabilities}

As mentioned above, we will be analyzing two depending stochastic processes: the 
random decrease of fitness and the random change of the mutation rate. Often, we will 
prove by drift analysis that the rate is drifting towards values that yield an 
almost-optimal fitness decrease. However, once the rate has drifted towards 
such values, we would also like the rates to stay in the vicinity of these 
values in subsequent steps. 
To this end, we apply the following theorem from \cite{KotzingLissWittFOGA15}.
Note that in the paper a slightly more general 
version including a self-loop probability is stated,
which we do not need here.

\begin{theorem}[Theorem~7 in \cite{KotzingLissWittFOGA15}]\label{lem:occupationLemma-cw}
Let a Markov process $(X_t)_{t \geq 0}$ on $\R_0^+$, where 
 $\lvert X_t - X_{t+1}\rvert \leq c$, 
with additive drift of at least $d$ towards $0$ be given (\ie, $E(X_t-X_{t+1}\mid X_t;X_t>0)\ge d$), 
starting at $0$ (i.e.\ $X_0=0$). 
 Then we have, for all $t\in \N$ and $b \in \R_0^+$,
\[
\Pr(X_t \ge b) \leq 
2 e^{\frac{2d}{3c}(1-b/c)}.
\]
\end{theorem}

We can readily apply this theorem in the following lemma that will be used throughout the paper to bound the rate~$r_t$.

\begin{lemma}
	\label{lem:occupation-rt}
	If there is a point $a\ge 4$ such that $\Prob(r_{t+1}<r_t\mid r_t>a)\ge 1/2+\eps$ for some constant~$\eps>0$, then 
	for all $t'\ge \min\{t\mid r_t\le a\}$ and all $b\ge 2$ it holds 
	$\Prob(r_{t'}\ge a\cdot 2^{b+1})\le 2e^{-2b\eps/3}$.
\end{lemma}

\begin{proof}
	Apply Theorem~\ref{lem:occupationLemma-cw} to the process $X_{t}:=\max\{0,\lfloor\log_2(r_t/a)\rfloor\}$.  
	Note that this process 
	is on $\N_0$, moves by an absolute value of at most $1$ and has drift $E(X_{t}-X_{t+1}\mid X_t;X_t>0)\ge 2\eps$. 
	We use  $c:=1$ and $d:=2\eps$ in the theorem and estimate $1-b\le -b/2$.
\end{proof}


%

\subsection{Useful Inequalities}

The following well-known estimates will be used regularly in this work. 

\begin{lemma}\label{lem:elem}
\begin{enumerate}
\item\label{it:elemlnx} For all $x > 0$, we have $\ln x \le x/e$.
\item\label{it:elemexp} For all $x \in \R$, we have  $1+x \le e^x$.
\item\label{it:elemexplow} For all $0\le x\le 2/3$, we have 
$1-x\ge e^{-x-x^2} \ge e^{-2x}$.
\end{enumerate}
\end{lemma}
\begin{proof}
	Since $(\ln x-x/e)'=1/x-1/e,x>0$, then $(e,0)$ is the maximum point of $\ln x-x/e,x>0$. Thus part \ref{it:elemlnx} holds. 
	
	We notice that $(1+x-e^x)'=1-e^x$, then $(0,0)$ is the maximal point of $1+x-e^x$. Therefore part \ref{it:elemexp} holds. 
	
	For the proof of part \ref{it:elemexplow}, the second inequality can be easily obtained, since $e^x$ monotonically increases and $-x-x^2\ge -2x$ when $0\le x\le 2/3$. To prove the first inequality, let $f(x)=1-x-e^{-x-x^2}$ for $0\le x\le 2/3$. Then $f'(x)=(1+2x)e^{-x-x^2}-1$ and $f''(x)=(1-4x-4x^2)e^{-x-x^2}$. We notice that $f''(x)$ has only one zero point between $0$ and $2/3$. Let $x_0$ denote the zero point. Thus $f''(x)$ is positive before $x_0$, and then becomes negative after $x_0$. Using the fact $f'(0)=0$ and $f'(2/3)<0$, we can easily see that $f'(x)$ first increases from $0$ to $f'(x_0)>0$, and then decreases below $0$. This means the minimum of $f(x)$ is attained at $0$ or $2/3$. Due to the fact that $f(0)=0$ and $f(2/3)>0$, the first inequality in part \ref{it:elemexplow} holds.
\end{proof}

\section{Proof Overview}
\label{sec:proof-overview}

Since the following runtime analysis of our self-adjusting \oplea on the \onemax function is slightly technical, let us outline the main proof ideas here in an informal manner. Let always $x$ denote the parent individual of the current iterations and $k := f(x)$ its fitness distance from the optimum (recall that we are minimizing the \onemax function, hence the fitness distance equals the objective function value which in turn is the Hamming distance from the optimum $x^* = (0, \dots, 0) \in \{0,1\}^n$ and thus the number of ones in $x$).

The outer proof argument is variable drift, that is, for each fitness value $k$ we estimate the expected fitness gain (``progress'') in an iteration starting with a parent individual $x$ with $f(x) = k$ and then we use the variable drift theorem (Theorem~\ref{theo:variable-upper}) to translate this information on the progress into an expected runtime (number of generations until the optimum is found). 

To obtain sufficiently strong lower bounds on the expected progress, we need to argue that the rate self-adjustment sufficiently often sets the current rate to a sufficiently good value. This is the main technical difficulty as it needs a very precise analysis of the quality of the offspring in both subpopulations. Since this requires different arguments depending on the current fitness distance $k$, we partition the process into three  regimes.

In the \emph{far region} covering the fitness distance values $k \ge n / \ln(\lambda)$, we need a rate~$r$ that is at least almost logarithmic in $\lambda$ to ensure that the average progress is high enough. Note that to gain the $\Theta(n)$ fitness levels from the initial value of approximately $n/2$ to $n / \ln \lambda$ in at most a time of $O(n / \log \lambda)$, we need an average progress of $\Omega(\ln \lambda)$ per iteration. We shall not be able to obtain this progress throughout this region, but we will obtain an expected progress of \[\Omega\bigg(\frac{\log \lambda}{\log\frac{en}{k}}\bigg)\] in an iteration with initial fitness distance $k$. This will be sufficient to reach a fitness distance of $n / \log \lambda$ in the desired $O(n / \log \lambda)$ iterations. 

As said, the main difficulty is arguing that the self-adjusting mechanism keeps the rate sufficiently often in the range we need, which is roughly 
\[\bigg[\frac{1}{2 \ln \frac{en}{k}} \, \ln \lambda, \frac{4n^2}{(n-2k)^2} \, \ln \lambda \bigg]\]
for all $k \in [n / \ln \lambda, n/2-1]$. Note that these range boundaries, in particular the upper one, depend strongly on $k$. Hence for arguing that our rate is sufficiently often in this range, we need to consider both the rate changes from the self-adjustments and the changing $k$-value. In this region we profit from the fact that we work with relatively high rates, which lead to strong concentration behaviors. This allows to argue that, for example, exceeding the upper boundary already by small constant factors is highly unlikely.

The \emph{middle region} covering the fitness distances $k \in [n / \lambda, n/ \ln \lambda]$ is small enough so that we do not require the algorithm to find a near-optimal rate very often. In fact, it suffices that the rate is below $\frac 12 \ln \lambda$ with constant probability. This ensures an expected progress of at least $\min\{\frac 18, \sqrt{\lambda}k / 32n\}$, which is sufficient to traverse also this region in time $O(n / \log \lambda)$. Consequently, in this region we only need to argue that the rate does not become too large, whereas there is no lower bound on $r$ which we require. Further, the upper bound of $\tfrac 12 \ln \lambda$ is large enough so that strong concentration arguments can be exploited, which show that deviations above the upper bound are highly unlikely. Since the upper bound does not change over time, we can now conveniently use an occupation probability argument to show that the rate with constant probability is only a small constant factor over the desired range, which is enough since the random fluctuations of the rate with constant probability reduce the rate further to the desired range. 

In the \emph{near region} covering the remaining fitness distance values $k \le n / \lambda$, the parent individual is already so close to the optimum that any rate higher than the minimal rate $r=2$ is sub-optimal. Hence in this region, we know precisely the optimum value for the rate. Nevertheless, it is not very easy to show that the subpopulation using the smaller rate $r/2$ has a higher chance of containing the new parent individual. Due to the small rates predominant in this region and the fact that progress generally is difficult (due to the proximity to the optimum), often both subpopulations will contain best offspring (which are in fact copies of the parent). Hence the typical reason for the winning offspring stemming from the smaller-rate subpopulation is not anymore that the other subpopulation  contains only worse individuals, but only that the smaller-rate subpopulation contains more best offspring. 

By quantifying this effect we establish a drift of the rate down to its minimum value~$2$ in the near region, and 
another occupation probability argument is used to show that the rate with sufficiently high probability will take this value in subsequent sets.
Nevertheless, the concentration of the rate around this target value is weaker than in the other regions, and additional care has to be taken to handle iterations 
in the near region in which the unlikely, 
but still possible event occurs that the rate exceeds~$\ln \lambda$. This value is a critical threshold in the near region since 
rates larger than $\ln \lambda$  will typically make all offspring worse than the parent.
As an additional obstacle, there is a small interval of rates above the threshold in which we cannot show a drift of the rate to its minimum. 
To show that this small interval nevertheless is left towards its lower end with probability~$\Omega(1)$, the occupation probability argument is combined with a potential function whose shape exploits the random adjustments that the \oplea performs with 50\% probability.

\section{Far Region}
\label{sec:far-region-rate-drift}

In this first of three technical sections, we analyze the optimization behavior of our self-adjusting \oplea in the regime where the fitness distance $k$ is at least $n/\ln\lambda$. Since we are relatively far from the optimum, it is relatively easy to make progress. On the other hand, this regime spans the largest number of fitness levels (namely $\Theta(n)$), so we need to exhibit a sufficient progress in each iteration. Also, this is the regime where the optimal mutation rate varies most drastically. Since it is not important for the following analysis, we remark without proof that the optimal rate is $n$ for $k \ge n/2 + \omega(\sqrt{n \log\lambda})$, $(1+o(1)) n/2$ for $k = n/2 \pm o(\sqrt{n \log\lambda})$, and then quickly drops to $r = \Theta(\log \lambda)$ for $k \le n/2 - \eps n$. Despite these difficulties, our \oplea manages to find sufficiently good mutation rates to be able to reach a fitness distance of $k = n/\ln\lambda$ in an expected number of $O(n / \log \lambda)$ iterations.

\begin{lemma}\label{lem:fardrift}
	Let $n$ be sufficiently large and $0<k<n/2$. We define $c_1(k)= (2\ln(e n/k))^{-1}$ and  $c_2(k)=4n^2/(n-2k)^2$.
	\begin{enumerate}
		\item \label{far-rate-inc}If $n/\ln\lambda\le k$ and $r \le c_1(k)\ln\lambda$, then the probability that a best offspring has been created with rate $2r$ is at least $0.64$.
		
		\item \label{far-rate-dec}Let $\lambda\ge 100$. If $ c_2(k)\ln\lambda\le r\le n/4$, then the probability that all best offspring have been created with rate $r/2$ is at least $0.51$.
		
		\item \label{far-rate-worse}If $r\ge 2(1+\gamma) c_2(k)\ln\lambda$, then the probability that all best offspring are worse than the parent is at least $1-\lambda^{-\gamma}$. 
	\end{enumerate}
\end{lemma}

\begin{proof}
	To prove part \ref{far-rate-inc}, let $q(k,i,r)$ and $Q(k,i,r)$ be the probabilities that standard bit mutation with mutation rate $p=r/n$ creates from a parent with fitness distance $k$ an offspring with fitness distance exactly $k-i$ and at most $k-i$. Then
	\[
	q(k,i,r)=\sum_{j=0}^{k-i}\binom{k}{i+j}\binom{n-k}{j}\left(\frac{r}{n}\right)^{i+2j}\left(1-\frac{r}{n}\right)^{n-i-2j}
	\]
	and $Q(k,i,r)=\sum_{j=i}^{k}q(k,j,r)$. We first show that the probability of not achieving $i\ge 2r$ is less than $0.2$. This is because for large enough $n$ and $r\ge 2$ we have $(1-o(1))\binom{k}{2r}\ge (1-o(1))4!(k/(2r))^{2r}>4(k/(2r))^{2r}$ and for $r\le c_1(k)\ln\lambda=o(\sqrt{n})$ by Lemma~\ref{lem:elem}, part~\ref{it:elemexplow} we have $(1-2r/n)^{n}\ge e^{-2r-4r^2/n}= (1-o(1))e^{-2r}$, thus
	\begin{align*}
		&Q(k,2r,2r)\ge q(k,2r,2r)\ge\binom{k}{2r}\left(\frac{2r}{n}\right)^{2r}\left(1-\frac{2r}{n}\right)^{n}\\&\ge 4\left(\frac{k}{2r}\right)^{2r}\left(\frac{2r}{n}\right)^{2r} e^{-2r}\ge 4\left(\frac{k}{en}\right)^{2r}
		\ge 4\left(\frac{k}{en}\right)^{2c_1(k)\ln\lambda}= \frac{4}{\lambda}.
	\end{align*}
	Considering $\lambda/2$ offspring using rate $2r$, the probability that none of them achieves a fitness improvement of at least $2r$ is less than $(1-4/\lambda)^{\lambda/2}<\exp(-4/2)<0.2$. We next argue that an offspring having a progress of $2r$ or more with good probability comes from the $2r$-subpopulation.
	We notice that 
	\[
	\frac{q(k,i,2r)}{q(k,i,r/2)}\ge 4^i \left(\frac{1-2r/n}{1-r/(2n)}\right)^{n}\ge \big(1-o(1)\big)4^ie^{-1.5r}.
	\]
	Thus for $r\ge 2$ and $i\ge 2r$ we have
	\[
	\frac{q(k,i,2r)}{q(k,i,r/2)}\ge\big(1-o(1)\big)\left(\frac{4}{e}\right)^{2r}>4.
	\]
	Therefore if the best progress among all $\lambda$ offspring is $i$ and $i\ge 2r$, the conditional probability that an offspring having fitness distance $k-i$ is generated with rate $2r$ is at least $q(k,i,2r)/(q(k,i,2r)+q(k,i,r/2))\ge 4/5$. In total, the probability that a best offspring has been created with rate $2r$ is at least
	\[
	(1-0.2)\cdot 4/6=0.64.
	\]

	For the proof of part~\ref{far-rate-dec}, let $\lambda\ge 100$ and $c_2(k)\ln\lambda\le r \le n/4$. Our idea is to show a probability of $o(1/\lambda)$ for the event that an offspring with rate $2r$ is not worse than the expected fitness of an offspring with rate $r/2$. Let $X(k,r)$ denote the random decrease of the fitness distance when applying standard bit mutation with probability $p=r/n$ to 
	an individual with $k$ ones. Then 
	\begin{gather*}
		E(X(k,r))=kp-(n-k)p=\frac{(2k-n)r}{n},
		\\\Var(X(k,r))=np(1-p)=r\left(1-\frac{r}{n}\right).
	\end{gather*}
	According to Bernstein's inequality (Theorem~\ref{thm:chernoff} \ref{bernstein}), for any $\Delta>0$ we have
	\[
	\Pr(X\ge E(X)+\Delta)\le\exp\left(\frac{-\Delta^2}{2\Var(X)+2\Delta/3}\right)\le\exp\left(\frac{-\Delta^2}{2\Var(X)+2\Delta}\right)
	\]
	We apply this bound with $X=X(k,2r)$ and $\Delta=E(X(k,r/2))-E(X(k,2r))=(n-2k)1.5r/n>0$. Then,
	\begin{align*}
		& \Pr(X(k,2r)\ge E(X(k,r/2)))\le\exp\left(\frac{-\Delta^2}{2\Var(X(k,2r))+2\Delta}\right)\\
		&=\exp\left(\frac{-9(n-2k)^2r}{4n(7n-8r-6k)}\right)
		\le\exp\left(-\frac{9(n-2k)^2c_2(k)\ln\lambda}{28n^2}\right)=\frac{1}{\lambda^{9/7}}.
	\end{align*}
	We notice that we have $7n-8r-6k>7n-4n-3n=0$ in the second inequality.
	From a union bound we see that with probability less than $\lambda^{-9/7}(\lambda/2)<100^{-2/7}/2<0.14$, the best offspring created with rate $2r$ is at least as good as the expected fitness of an individual created with rate $r/2$. 
	
	We estimate the probability that the best offspring using rate $r/2$ has a fitness distance at most $E(X(k,r/2))$.  Let $y$ be an offspring obtained from $x$ via standard bit mutation with mutation rate $r/(2n)$.  Let $X^+$ and $X^-$ be the number of one-bits flipped and zero-bits flipped, respectively.
	\begin{align*}
	X^+ &:= |\{i \in \{1, \dots, n\} \mid x_i=1 \wedge y_i =0\}|,\\
	X^- &:= |\{i \in \{1, \dots, n\} \mid x_i=0 \wedge y_i =1\}|. 
	\end{align*}
	Both $X^+$ and $X^-$ are binomially distributed and $X(k,r/2)=X^+-X^-$. We aim at a lower bound for $\Pr(X(k,r/2)\ge E(X(k,r/2)))$. We have $\Pr(X^+\ge E(X^+)-1)\ge 1/2$, since the median of $X^+$ is between $\lfloor E(X^+)  \rfloor$ and $\lceil E(X^+)\rceil$ by \cite{Kaas80}.
	It remains to bound $\Prob(X^-\le E(X^-)-1)$. We notice that $E(X^-)=(n-k)r/(2n)>\ln\lambda>1$ and $E(X^-)\le (n-k)/4$.
	Applying Theorem~12 in \cite{Doerr18exceedexp} to binomial random variable $\tilde{X}^-:=(n-k)-X^-$, we obtain
	$\Prob(X^-\le E(X^-)-1)=\Prob(\tilde{X}^-\ge E(\tilde{X}^-)+1)\ge 0.037.$
	Therefore 
	\begin{align*}
		\Pr(X(k,r/2)\ge E(X(k,r/2)))&\ge\Prob(X^-\le E(X^-)-1)\Prob(X^+\ge E(X^+)-1)\\
		&\ge 0.037\cdot 0.5=0.0185.
	\end{align*}
	For $\lambda/2$ offspring using rate $r/2$, the probability that the best one has a fitness distance at most $E(X(k,r/2))$ is more than $1-(1-0.0185)^{\lambda/2}\ge 1-0.9815^{50}\ge 0.6$. Therefore with probability at least $0.6\cdot(1-0.14)>0.51$, all best offspring are from  the $r/2$-subpopulation. This proves the second statement of the lemma.
	
	For proof of part \ref{far-rate-worse}, let $r\ge 2(1+\gamma)c_2(k)\ln\lambda$. An offspring created with mutation rate $r/n$ is at least as good as its parent if and only if $X(k,r)\ge 0$. By using Bernstein's inequality (Theorem~\ref{thm:chernoff}  \ref{bernstein}) with $X=X(k,r)$ and $\Delta=-E(X(k,r))$, we have
	\begin{align*}
	\Pr(X(k,r)\ge 0)&\le\exp\left(-\frac{E(X(k,r))^2}{2\Var(X(k,r))+2E(X(k,r))}\right)\\
	&=\exp\left(-\frac{(n-2k)^2r}{2n(2n-2k-r)}\right)
	\le\exp\left(-\frac{(n-2k)^2r}{4n^2}\right).
	\end{align*}
	Since $r\ge 2(1+\gamma) c_2(k)\ln\lambda$, the corresponding probabilities  for rate $r/2$ and $2r$ are at most  $1/\lambda^{1+\gamma}$ and $1/\lambda^{4+4\gamma}$, respectively. By a union bound, with probability at most $1/\lambda^{\gamma}$, the best offspring is at least as good as its parent. This proves part~\ref{far-rate-worse}.
\end{proof}

The lemma above shows that the rate $r$ is subject to a constant drift towards the interval $[c_1(k) \ln n, c_2(k) \ln n]$. Unfortunately, we cannot show that we obtain a sufficient fitness progress for all $r$-values in this range. However, we can do so for a range smaller only by constant factors. This is what we do now (for large values of $k$) and in Lemma~\ref{lem:far-region-drift} (for smaller values of $k$). This case distinction is motivated by the fact that $c_2(k)$ becomes very large when $k$ approaches $n/2$. Having a good fitness drift only for such a smaller range of $r$-values is not a problem since the random movements of $r$ let us enter the smaller range with constant probability. This is what we will exploit in Theorem~\ref{thm:expected-time-far} and its proof.

 Let $2\le r \le n/4$ be the current rate and let $\tilde r \in \{r/2,2r\}$. Let $\tilde{\Delta}(\lambda/2,k,\tilde{r})$ denote the fitness gain of the best of $\lambda/2$ offspring generated with rate $\tilde{r}$ from a parent $x$ with fitness distance $k := \onemax(x)$ and the parent itself. More precisely, let $x^{(i)}, i \in \{1,\dots,\lambda/2\}$, be independent offspring generated from $x$ by flipping each bit independently with probability $\tilde{r}/n$.  Then the random variable $\tilde{\Delta}(\lambda/2,k,\tilde{r})$ is defined by $\max\{0,k - \min\{\om(x^{(i)}) \mid i \in \{1,\dots,\lambda/2\}\}\}$. We use $\Delta:=\max\{\tilde{\Delta}(\lambda/2,k,r/2),\\\tilde{\Delta}(\lambda/2,k,2r)\}$ to denote the fitness gain in a iteration which uses $x$ as parent and $r$ as mutation rate.
 

We next show that a region contained in $[c_1(k)\ln\lambda,c_2(k)\ln\lambda]$ provides at least a logarithmic (in $\lambda$) drift on fitness. 

\begin{lemma}\label{lem:super-far-region-rate-drift}
	Let $n$ be sufficiently large, $2n/5\le k<n/2$ and $\lambda\ge e^5>148$. If $2\le \ln(\lambda)\le r\le \min\{n^2\ln(\lambda)/(25(n-2k)^2),n/4\}$, then $E(\Delta\mid k) \ge 10^{-3}\ln\lambda$.
\end{lemma}
\begin{proof}
	We first notice that $\ln(\lambda)\ge5$ and $n^2/(25(n-2k)^2)\ge 1/(25\cdot0.2^2)=1$ for all $2n/5\le k<n/2$. We aim to prove $\tilde{\Delta}(\lambda/2,k,\tilde{r})\ge 10^{-3}\ln\lambda$ with $\tilde{r}=r/2$.
	Let $X^+$ and $X^-$ be the number of one-bits flipped and zero-bits flipped, respectively, in an offspring using rate $p=\tilde{r}/n$.
	$X^+$ and $X^-$ follow binomial distributions $\Bin(k,p)$ and $\Bin(n-k,p)$, respectively. Let $u:=E(X^+)=kp$. Then $u\le k/8$ and $u\ge (k/n)\ln(\lambda)/2\ge 0.4\cdot0.5\cdot \ln\lambda=0.2\ln\lambda\ge 1$. Furthermore, let 
	\begin{align*}
		&B(x):=\Pr(X^+=x)=\binom{k}{x}p^x(1-p)^{k-x}\text{ for all } x\in\{0,1,\dots,k\},\\
		&F(y):=\Pr(X^+\ge y)=\sum_{i=\lceil y\rceil}^{k}B(i)\text{ for all } y\in[0,k].
	\end{align*}
	Applying Theorem~10 in \cite{Doerr18exceedexp}, we obtain $F(u)=\Pr(X^+\ge E(X^+))\ge 1/4$. Similarly $\Pr(X^-\le E(X^-))\ge 1/4$. We prove that for $\delta:=E(X^-)-E(X^+)+0.05\ln\lambda=(n-2k)p+0.05\ln\lambda$, we have $F(u+\delta)\ge \lambda^{-0.8}/9$, and thus $\Pr(X^+-X^-\ge 0.05\ln\lambda)\ge \lambda^{-0.8}/36$. Since for any $x\in\Z_{\ge u}$ we have
	\[
	\frac{B(x+1)}{B(x)}=\frac{k-x}{x+1}\cdot\frac{p}{1-p}\le \frac{u-up}{u-up+1-p}<1,
	\] 
	we obtain $B(\lceil u\rceil)>B(\lceil u\rceil+1)>\dots>B(k)$, and thus $F( u+ \delta)\ge   \lceil\delta\rceil B(\lceil u+2 \delta\rceil)$ as well as $ \lceil \delta\rceil  B(\lceil u\rceil)\ge F(u)-F(u +\delta)$. We see that
	\begin{align*}
	\frac{B(\lceil u+2 \delta\rceil)}{B(\lceil u\rceil)}&=\frac{(k-\lceil u\rceil)\cdots(k-\lceil u+2 \delta\rceil+1)}{(\lceil u\rceil+1)\cdots(\lceil u+2\delta\rceil)}\cdot\frac{p^{2\delta}}{(1-p)^{2\delta}}\\
	&\ge  \left(\frac{k-u-2\delta}{k(1-p)}\right)^{2\delta}\frac{u^{2\delta}}{\lceil u+1\rceil\cdots\lceil u+2\delta\rceil}\ge  \left(1-\frac{2\delta}{3u}\right)^{2\delta}\frac{u^{2\delta}}{\lceil u+1\rceil\cdots\lceil u+2\delta\rceil},
	\end{align*}
	where we used $k(1-p)=k-u\ge 8u-u=7u$.
	Using a sharp version of Stirling's approximation due to Robbins~\cite{Robbins55}, we compute
	\begin{align*}	
	&\lceil u+2\delta\rceil!\le \sqrt{2\pi\lceil u+2\delta\rceil}\left(\frac{\lceil u+2\delta\rceil}{e}\right)^{\lceil u+2\delta\rceil}\exp\left(\frac{1}{12(\lceil u+2\delta\rceil)}\right),\\
	&\lceil u\rceil!\ge \sqrt{2\pi \lceil u\rceil}\left(\frac{\lceil u\rceil}{e}\right)^{\lceil u\rceil}\exp\left(\frac{1}{12\lceil u\rceil+1}\right),\\
	&\frac{1}{\lceil u+1\rceil\cdots\lceil u+2\delta\rceil}=\frac{\lceil u\rceil!}{\lceil u+2\delta\rceil!}\ge\sqrt{\frac{\lceil u\rceil}{\lceil u+2\delta\rceil}}\frac{\lceil u\rceil^{\lceil u\rceil}e^{2\delta}}{\lceil u+2\delta\rceil^{\lceil u+2\delta\rceil}}\ge \frac{ \sqrt{\frac{u}{ u+2\delta+1}}u^{ u}e^{2\delta}}{(u+2\delta+1)^{ u+2\delta+1}}.
	\end{align*}
	We notice that $u\ge (k/n)\ln(\lambda)/2\ge 0.4\cdot0.5\cdot \ln\lambda=0.2\ln\lambda$ and $$\delta=(n-2k)p+0.05\ln\lambda\le 0.2np+0.05\ln\lambda\le 0.5u+0.05\ln\lambda<u.$$
	Thus $2\delta/(3u)\le 2/3$ and $2\delta/(u+2\delta)\le 2\delta/(\delta+2\delta)=2/3$. By Lemma~\ref{lem:elem}  \ref{it:elemexplow} we have $\ln(1-x)\ge -x-x^2\ge -2x$ for $0\le x\le 2/3$. Hence
	\begin{align*}
	\frac{B(\lceil u+2\delta \rceil)}{B(\lceil u\rceil)}&\ge \sqrt{\frac{u}{u+2\delta+1}}\left(1-\frac{2\delta}{7u}\right)^{2\delta}\frac{u^{u+2\delta}e^{2\delta}}{(u+2\delta+1)^{u+2\delta+1}}\\
	&\ge  \sqrt{\frac{1}{4}}\exp\left(2\delta\ln\left(1-\frac{2\delta}{7u}\right)+(u+2\delta+1)\ln\left(1-\frac{2\delta}{u+2\delta+1}\right)+2\delta\right)\\
	&\ge  \frac{1}{2}\exp\left(-\frac{2(2\delta)^2}{7u}-\frac{(2\delta)^2}{u+2\delta+1}\right)\ge\frac{1}{2}\exp\left(-\frac{8\delta^2}{7u}-\frac{4\delta^2}{u}\right)\\
	&=\frac{1}{2}\exp\left(\frac{-6\delta^2}{u}\right)\ge \frac{1}{2}\lambda^{-0.8},
	\end{align*}
	where in the last inequality, using $\ln(\lambda)/2\le \tilde{r}\le n^2\ln(\lambda)/(50(n-2k)^2)$, we estimate
	\begin{align*}
	\frac{6\delta^2}{u}&=\frac{6((n-2k)\tilde{r}/n+0.05\ln\lambda)^2}{k\tilde{r}/n}\\	 
	&= \frac{6n}{k}\left(\frac{(n-2k)^2\tilde{r}}{n^2}+\frac{2(n-2k)0.05\ln\lambda}{n}+\frac{(0.05\ln\lambda)^2}{\tilde{r}}\right)\\
	&\le \frac{6}{0.4}\left(\frac{\ln\lambda}{50}+0.2\cdot 0.1\ln\lambda+\frac{0.05^2\ln\lambda}{0.5}\right)<0.8\ln\lambda.
	\end{align*}
	Recalling $F(u+\delta)\ge \lceil\delta\rceil B(\lceil u+2 \delta\rceil)$ and $\lceil\delta\rceil B(\lceil u\rceil)\ge F(u)-F(u+\delta)$, we compute
	 $$F(u+\delta)\ge \lceil\delta\rceil B(\lceil u+2 \delta\rceil)\ge\lceil\delta\rceil (0.5\lambda^{-0.8})B(\lceil u+2 \delta\rceil)\ge(0.5\lambda^{-0.8})(F(u)-F(u+\delta)).$$ 
	Since $F(u)\ge 0.25$ and $\lambda>148$, we obtain 
	\[
	F(u+\delta)\ge \frac{0.5\lambda^{-0.8}F(u)}{1+0.5\lambda^{-0.8}}\ge \frac{0.5\cdot 0.25\cdot\lambda^{-0.8}}{1+0.5\cdot148^{-0.8}}>\frac{\lambda^{-0.8}}{9}.
	\]
	Using $\Pr(X^+-X^-\ge 0.05\ln\lambda )\ge \lambda^{-0.8}/36$, we bound
	\begin{align*}
	\Pr(\tilde{\Delta}(\lambda/2,k,r/2)\ge 0.05\ln\lambda\mid k)&\ge 1-(1-\lambda^{-0.8}/36)^{\lambda/2}\\&\ge 1-(1-148^{-0.8}/36)^{148/2}>0.02.
	\end{align*}
	Finally $E(\Delta\mid k)\ge 0.02\cdot 0.05\ln\lambda=10^{-3}\ln\lambda$.
\end{proof}

We now extend the lemma to the whole region of $n/\ln\lambda\le k< n/2$. If $k<2n/5$ the situation becomes easier because $4\le c_2(k)<100$ and every $r$ in the smaller range $[c_1(k)\ln\lambda,  \ln(\lambda)/2]$ provides at least an expected fitness improvement that is logarithmic in $\lambda$. Together with the previous lemma, we obtain the following statement for the drift in the whole region $n/\ln\lambda\le k<n/2$. 

\begin{lemma}\label{lem:far-region-drift}
	Let $n$ be sufficiently large, $n/\ln\lambda\le k< n/2$ and $\lambda\ge e^5$. If $c_1(k)\ln\lambda\le r\le c_2(k)\ln(\lambda)/100$ with $c_1(k), c_2(k)$ defined as in Lemma~\ref{lem:fardrift}, then 
	\[
	E(\Delta\mid k)\ge 10^{-3}\ln(\lambda)/\ln(e n/k).
	\]
	
\end{lemma}

\begin{proof}
	If $r> \ln(\lambda)$, then $c_2(2n/5)=100$ implies $k> 2n/5$ and the claim follows from Lemma~\ref{lem:super-far-region-rate-drift}. Hence let us assume $r\le \ln(\lambda)$ in the remainder and compute $\tilde{\Delta}(\lambda/2,k,\tilde{r})$ with $\tilde{r}=r/2$.
	
	We consider the probability $Q(k,i,\tilde{r})$ of creating from a parent with distance $k$ an offspring with fitness distance at most $k-i$ via standard bit mutation with mutation rate $p=\tilde{r}/n$.  
	Let $i:= \max\{1,\lfloor c_1(k)\ln(\lambda)/2\rfloor\}\le \max\{1,r/2\}=\tilde{r}$. If $\lfloor c_1(k)\ln(\lambda)/2\rfloor<1$,
	using $(1-p)^n\ge e^{-pn-p^2n}\ge (1-o(1))e^{-\tilde{r}}$ by Lemma~\ref{lem:elem}\ref{it:elemexplow}, we compute
	\[
	Q(k,1,\tilde{r})\ge kp(1-p)^n
	\ge \frac{k\tilde{r}}{n }(1-p)^n\ge \frac{(1-p)^n}{\ln\lambda}\ge\frac{(1-o(1))e^{-\ln(\lambda)/2}}{\ln\lambda}\ge \frac{1}{\lambda}.
	\]
	Otherwise if $\lfloor c_1(k)\ln(\lambda)/2\rfloor>1$,
	using $\binom{k}{i}\ge (k/i)^{i}$, we obtain
	\begin{align*}
	Q(k,i,\tilde{r})& \ge\binom{k}{i}p^{i}(1-p)^n
	\ge \left(\frac{k}{i}\cdot\frac{\tilde{r}}{n}\right)^{i}e^{-\tilde{r}-p^2n}\ge\left(\frac{k}{n}\right)^{i}e^{-\tilde{r}-p^2n}\\
	&\ge \left(\frac{k}{en}\right)^{i}e^{-\tilde{r}}\ge \left(\frac{k}{en}\right)^{\frac{\ln\lambda}{4\ln(en/k)}}e^{-\tilde{r}}
	= e^{-\ln(\lambda)/4-\tilde{r}}\ge e^{-\ln \lambda}=\frac{1}{\lambda}.
	\end{align*}
	Hence, $\Pr(\tilde{\Delta}(\lambda/2,k,\tilde{r})\ge i\mid k)\ge 1-(1-1/\lambda)^{\lambda/2}\ge 1-e^{-1/2}> 0.3$. We notice that $i\ge c_1(k)\ln(\lambda)/4$, since $i=1\ge c_1(k)\ln(\lambda)/4$ when $c_1(k)\ln\lambda<2$ and $i=\lfloor c_1(k)\ln(\lambda)/2\rfloor\ge c_1(k)\ln(\lambda)/4$ when $c_1(k)\ln\lambda\ge 2$. Consequently 
	\[
	E(\Delta\mid k)> \Pr(\tilde{\Delta}(\lambda/2,k,\tilde{r})\ge i\mid k)\cdot i\ge 0.3\cdot c_1(k)\ln(\lambda)/4=(0.3/8)\ln(\lambda)/\ln(e n/k).
	\]
	\end{proof}

If we only consider generations that use a rate within the right region, we can bound the expected runtime to reach $k\le n/\ln\lambda$ by $O(n/\log \lambda)$  since the drift on the fitness is of order $\log\lambda$. The following theorem shows that the additional time spent with adjusting the rate towards the right region does not change this bound on the expected runtime.

\begin{theorem}
	\label{thm:expected-time-far}
	The \oplea with self-adjusting mutation rate reaches a \om-value of $k\le n/ \ln\lambda$ within an expected number of $O(n / \log\lambda)$ iterations. This bound is valid regardless of the initial mutation rate.
\end{theorem}

\begin{proof}
  Let us denote by $k(x)$ the fitness distance of a search point $x \in \{0,1\}^n$.
  
	We first argue that it takes an expected number of at most $O(\sqrt{n})$ iterations to reach a fitness distance of less than $n/2$. To this end, consider a parent $x$ with fitness distance $k(x) \ge n/2$. Let $2\le r \le n/4$ be the current rate and let $\tilde r \in \{r/2,2r\}$. Let $y$ be an offspring obtained from $x$ via standard bit mutation with mutation rate $\tilde r/n$. Let 
	\begin{align*}
	X^+ &:= |\{i \in \{1, \dots, n\} \mid x_i=1 \wedge y_i =0\}|,\\
	X^- &:= |\{i \in \{1, \dots, n\} \mid x_i=0 \wedge y_i =1\}|. 
	\end{align*}
Then the fitness improvement is $k(x)-k(y) = X^+ - X^-$ and both $X^+$ and $X^-$ follow binomial distributions. From elementary properties of the binomial distribution, see, e.g.,~\cite{Doerr18exceedexp}, we have $\Pr(X^+ \ge E(X^+) + 1) = \Omega(1)$ and $\Pr(X^- \le E(X^-)) = \Omega(1)$, and these are independent events. Hence with constant probability, we have 
\begin{align*}
k(y) &= k(x) - X^+ + X^- \le k(x) - E(X^+) - 1 + E(X^-) \\
&\le k(x) - 1 - (2k(x)-n)\tfrac{\tilde r}n \le k(x) - 1.
\end{align*}
	
	Clearly, for the best offspring $z$ out of the $\lambda$ offspring generated in this iteration, we have $k(z) \le k(y)$. Consequently, in each iteration starting with a parent~$x$ with $k(x) \ge n/2$, with constant probability we gain at least one fitness level. By the additive drift theorem (see Theorem~\ref{theo:variable-upper} and the subsequent discussion), from the initial random search point $x_0$ it takes an expected $O(\max\{0,k(x_0) - n/2 + 1\})$ iterations to reach a parent with fitness distance less than $n/2$.
	Since $X := k(x_0) \sim \Bin(n,1/2)$, we have 
	\begin{align*}
	E(\max\{0,X-n/2\}) &= \tfrac 12 E(|X - E(X)|) \le \tfrac 12 \sqrt{E((X- E(X))^2)} \\
	&= \tfrac 12 \sqrt{\Var(X)} = \tfrac 14 \sqrt n,
	\end{align*}
	where we first exploited the symmetry of $X$ and then the well-known estimate $E(Y)^2 \le E(Y^2)$ valid for all random variables $Y$, in particular, for $Y = |X - E(X)|$. By the law of total expectation, the expected time to reach a search point with $k$-value below $n/2$ is $O(\sqrt n)$.
	
	Without loss of generality we can now assume $k<n/2$ for the initial state. Our intuition is that once we begin to use a rate $r\in [(c_1(k)/2)\ln\lambda,c_2(k)\ln\lambda]$ at some distance level $k$, we will have a considerable drift on the \om-value and the strong drift on the rate keeps $r$ within or close to this interval. After we make progress and $k$ decreases to a new level, the corresponding $c_1$ and $c_2$ decrease, and the algorithm may take some time to readjust $r$ into new bounds.
	
	We consider the stochastic process $X_t=\lfloor\log_2(r_t)\rfloor$ and the current \om-value $Y_t$. According to Lemma~\ref{lem:fardrift} \ref{far-rate-inc} and \ref{far-rate-dec}, there exists $\eps=\Omega(1)$ such that
	\begin{align*}
		&\Pr(X_{t+1}-X_t\mid X_t;X_t < \log_2(c_1(K_t)\ln\lambda)-1)\ge \eps,\\
		&\Pr(X_t-X_{t+1}\mid X_t;X_t > \log_2(c_2(K_t)\ln\lambda))\ge \eps.
	\end{align*} 
		
Let $k_0>k_1>\dots>k_N$ be all the different \om-values taken by $Y_t$ until for the first time $Y_t\le n/\ln\lambda$. 
By the additive drift theorem, it takes at most $O(\log n)$ iterations to have $r_t\in [(c_1(k_0)/2)\ln\lambda,c_2(k_0)\ln\lambda]$, regardless of how we set the initial rate. Since  $c_1(Y_t)$ is non-increasing,  $r_t\ge c_1(Y_t)\ln(\lambda)/2$  implies $r_t\ge c_1(Y_{t+1})\ln(\lambda)/2$. Thus, no readjustments are necessary to satisfy the lower bound $r_t\ge c_1(Y_t)\ln(\lambda)/2$ once we have $r_t\ge c_1(k_0)\ln(\lambda)/2$. We now regard the upper bound condition $r_t\le c_2(Y_t)\ln(\lambda)$.
Let $(k_{s_i})$ be the subsequence consisting of all $k_{s_i}$ such that $\{t\mid r_t\le c_2(Y_t)\ln(\lambda),Y_t=k_{s_i}\}\neq \emptyset$.
If $k_i$ is in the subsequence, once $r_t\le c_2(Y_t)\ln(\lambda)$ is achieved, the runtime until the first time $Y_t=k_{i+1}$ can be computed using the occupation lemma and the variable drift theorem. After that, it takes at most $O(\log(c_2(k_i)/c_2(k_{i+1})))$ more iterations to readjust $r_t$ from $c_2(k_i)\ln(\lambda)$ to $c_2(k_{i+1})\ln(\lambda)$ by the additive drift theorem. Similar arguments hold for $k_{i+1}$ if it  is also in the subsequence. Otherwise let $k_j<k_{i+1}$ denote the next \om-values after $k_i$ in the subsequence. By definition, the fitness distance decreases from $k_{i+1}$ to $k_j$ before the rate being adjusted into the corresponding range. This means that improvements in distance are made during readjustment. Analogous to above, the expected (readjusting) time to decrease the rate from $c_2(k_{i+1})\ln(\lambda)$ to $c_2(k_{j})\ln(\lambda)$ is $O(\log(c_2(k_{i+1})/c_2(k_{j})))$. Therefore, the expected number of total readjusting time is at most

\[
\sum_{ i=1}^{s_{i+1}\le N}O\left(\log\left(\frac{c_2(k_{s_i})}{c_2(k_{s_{i+1}})}\right)\right)=O\left(\log\left(\frac{c_2(k_0)}{c_2(k_N)}\right)\right)=O(\log n).
\]

When computing the expected number of non-adjusting iterations until $Y_t=k_N$, we choose a constant $b\in\N_{\ge 2}$ large enough such that $2e^{-2b\eps/3}\le 1/2-\delta/2$ holds for some positive constant $\delta>0$. When $r_t\in [(c_1(k)/2)\ln\lambda,c_2(k)\ln\lambda]$ and $Y_t=k$ for some $t$, then by Lemma~\ref{lem:occupation-rt}, we obtain 
\begin{align*}
\Prob(r_{t'}\ge 2^{b+1}c_2(k)\ln\lambda)\le 2e^{-2b\eps/3} \text{ and } \Prob(r_{t'}\le 2^{-b-2}c_1(k)\ln\lambda)\le 2e^{-2b\eps/3} \text{ for }t'\ge t.
\end{align*}
Hence, we obtain that $r_{t'}\in[2^{-b-2}c_1(k)\ln\lambda,2^{b+1}c_2(k)\ln\lambda]$ happens with probability at least $\delta$.
We see that there are at most $\lceil\log_2(100)\rceil$ steps between being in the range $[(c_1(k)/2)\ln\lambda, c_2(k)\ln\lambda]$ and being in the smaller range $[c_1(k)\ln\lambda, c_2(k)\ln(\lambda)/100]$ which is described in Lemma \ref{lem:far-region-drift}. If $r_t\in[2^{-b-2}c_1(k)\ln\lambda,2^{b+1}c_2(k)\ln\lambda]$, it takes at most a constant number of iterations $\alpha$ in expectation to reach $[c_1(k)\ln\lambda, c_2(k)\ln(\lambda)/100]$ because our mutation scheme employs a $50\%$ chance to perform a random step of the mutation rate. Based on Lemma~\ref{lem:far-region-drift}, the fitness drift at distance $k$ of all rates within the narrow region is at least $10^{-3}\ln(\lambda)/\ln(en/k)$. This contributes to an average drift of at least
\[
10^{-3}\cdot \frac{\ln(\lambda)}{\ln(en/k)}\cdot\frac{\delta}{1+\alpha}=\Omega\left(\frac{\log(\lambda)}{\log(n/k)}\right)
\] 
for all random rates at distance $k$. Using the variable drift theorem (Theorem~\ref{theo:variable-upper}), we estimate the runtime as

\begin{align*}
&O\left(\frac{\log(\log\lambda)}{\log\lambda}+\int_{n/\log\lambda}^{n/2}\frac{\log(n/k)\mathrm{d}k}{\log\lambda}\right)
=O\left(\frac{1}{\log\lambda}\int_{n/\log\lambda}^{n/2}\bigg(\log(n)-\log (k)\bigg)\mathrm{d}k\right)\\
=&O\left(\frac{1}{\log\lambda}\bigg(\log(n)k-k\log(k)+\log(k)\bigg)\bigg|_{n/\log\lambda}^{n/2}\right)\\
=&O\left(\frac{1}{\log\lambda}\left(\log(n)\left(\frac{n}{2}-\frac{n}{\log\lambda}\right)-\frac{n}{2}\log\left(\frac{n}{2}\right)+\frac{n}{\log\lambda}\log\left(\frac{n}{\log\lambda}\right)+\log\left(\frac{n/2}{n/\log\lambda}\right)\right)\right)\\
=&O\left(\frac{1}{\log\lambda}\left(\frac{n\log2}{2}-\frac{n\log(\log\lambda)}{\log\lambda}+\log(\log\lambda)\right)\right)
=O\left(\frac{n}{\log\lambda}\right).
\end{align*}
We notice that the expected runtime for adjusting $r_t$ is $O(\log n)=O(n/\log\lambda)$. Therefore, the total runtime is $O(n/\log\lambda)$ in expectation.

\end{proof}

\section{Middle Region}

In this section we estimate the expected number of generations until the number of one-bits
 has decreased from $k\le n/\ln\lambda$ to $k\le n/\lambda$. 
We first claim that the right region for $r$ is $1\le r\le \ln(\lambda)/2$. Hence, 
the \oplea is not very sensitive to the choice of~$r$ here. Intuitively, 
this is due to the fact that a total fitness improvement of only $O(n/\log\lambda)$ suffices  
to cross the middle region, whereas an improvement of $\Omega(n)$ is needed for the 
far region.  

We estimate the drift of the fitness in Lemma~\ref{lem:middle-region-drift} and apply 
that result afterwards to estimate the number of generations to cross the region.

\begin{lemma}\label{lem:middle-region-drift}
Let $n/\lambda\le k\le n/\ln\lambda$, $\lambda\geq 26$
and $2\le r\le \ln\lambda$.
Then \[E(\Delta\mid k)\ge \min\left\{\frac{1}{8},\frac{\sqrt{\lambda} k}{32n}\right\}.\]
\end{lemma}

\begin{proof}
We aim at computing $\tilde{\Delta}(\lambda/2,k,\tilde{r})$ with $\tilde{r}=r/2$.
The probability that no zero-bit flips in a single offspring stemming from rate~$\tilde{r}$ 
is $(1-\tilde{r}/n)^{n-k}\ge e^{-\tilde{r}}\ge 1/\sqrt{\lambda}$.
We regard the number $Z$ of offspring of rate~$\tilde{r}$ that have no flipped zeros.
The expectation $E(Z)$ is at least $(1/\sqrt{\lambda})\cdot(\lambda/2)=\sqrt{\lambda}/2$.
Applying Chernoff bounds (Theorem~\ref{thm:chernoff}),
we observe that $Z$ exceeds $\lambda_0:=\sqrt{\lambda}/4$
with probability at least $1-\exp(-\sqrt{\lambda}/16) > 1/4$
since $\lambda\geq 26$.
Assuming this to happen,
we look at the first $\lambda_0$ offspring without flipped zeros.
For $i\in\{1,\dots,\lambda_0\}$ let $X_i$ be the number of flipped ones
in the $i$-th offspring.
Then the $X_i$ are i.i.d.\ with $X_i\sim\Bin(k,\tilde{r}/n)$.
Let $X^*\coloneqq \max\{X_i\}$.
The expectation of~$X^*$ is analyzed in \citet[Lemma~4, part~2]{GiessenWittALGO16};  
more precisely the following statement is shown for constant~$\tilde{r}$:
\[
\text{If~} \frac{\lambda_0k\tilde{r}}{n}\ge \alpha \text{~then~} E(X^*)\ge \frac{\alpha}{1+\alpha}.
\]
Since $\tilde{r}$ is not necessarily assumed constant here, we generalize the result by closely following the proof in~\cite{GiessenWittALGO16}. 
Note that $X^*\ge 1$ if at least one of the offspring does not flip a zero-bit. Hence,
\begin{align*}
E(X^*) & \ge 1- \left(\left(1-\frac{\tilde{r}}{n}\right)^k\right)^{\lambda_0}
 \ge 1- \left(\left(1-\frac{\tilde{r}}{n}\right)^{\alpha n /(\lambda_0 \tilde{r})}\right)^{\lambda_0}\\
& \ge 1- \left(e^{-\alpha/\lambda_0}\right)^{\lambda_0} = 1-e^{-\alpha} \ge 1 - \frac{1}{1+\alpha} = \frac{\alpha}{1+\alpha},
\end{align*}
where the second inequality used the assumption $\lambda_0k\tilde{r}/n\ge \alpha$, the third one $e^{x}\ge 1+x$ for $x\in \R$.  
and the fifth one the equivalent to the previous inequality $e^{-x}\le 1/(1+x)$. Hence, 
Lemma~4, part~2 in \cite{GiessenWittALGO16} holds also for arbitrary~$\tilde{r}$.

Setting $\alpha\coloneqq \lambda_0k\tilde{r}/n$, we now distinguish between two cases. 
If $\alpha\ge 1$, we obtain $E(X^*)\ge \alpha/(1+\alpha)\ge 1/2$; 
otherwise we have $1+\alpha<2$, hence $E(X^*)\ge \alpha/2 = \lambda_0k\tilde{r}/(2n)$.
Thus,
\[E(X^*)\ge \min\{1/2,\sqrt{\lambda}k/(8n)\}.\]
Hence, using the law of total probability,
we obtain the lower bound on the drift for the middle region.
\end{proof}

We now use our result on the drift to estimate the time spent in this region. We notice that $c_2(k)=4+o(1)$ when $k=o(n)$. This means we will often have 
$r_t\in[2,\ln(\lambda)]$ which provides the drift we need.

\begin{theorem}\label{thm:expected-time-middle}
	Let $\lambda\ge 26$ and $\lambda=n^{O(1)}$. Assume $k\le n/\ln\lambda$ for the current \om-value of the self-adjusting \oplea. Then the expected number of generations until $k\le n/\lambda$ is $O( n/\log\lambda)$.
\end{theorem}

\begin{proof}
For $k\le n/\ln\lambda$ the upper bound from Lemma~\ref{lem:fardrift} is $c_2(k)=4/(1-2/\ln\lambda)<11$.
According to the lemma
we have for $X_t\coloneqq \lceil\log_2 r_t\rceil$ that
$E(X_t-X_{t+1}\mid X_t;X_t > \log_2(c_2(k)\ln\lambda))\ge 2\eps=\Omega(1)$.
The additive drift theorem yields that in $O(\log n)$ time
we have $r_t\le c_2(k)\ln\lambda$.
We choose $b$ large enough
such that $2e^{-2b\eps/3}\le 1-\delta$
holds for some positive constant $\delta>0$
and note that $b$ is constant.
Applying Lemma~\ref{lem:occupation-rt}
we obtain $\Prob(r_{t}\ge 2^bc_2(k)\ln\lambda)\le 2e^{-2b\eps/3}$
and $r_t\le 2^bc_2(k)\ln\lambda$ happens with probability at least $\delta$.
Once $r_t<2^bc_2(k)\ln\lambda<2^b(11\ln\lambda)$
it takes at most a constant number of iterations $\alpha$ in expectation
to draw $r_t$ to $\ln(\lambda)$ or less.
According to Lemma~\ref{lem:middle-region-drift}
this ensures a drift of at least $(1/4)\min\{1/2,\sqrt{\lambda}k/(8n)\} \ge (1/32)\min\{1,\sqrt{\lambda}k/n\}$,
which implies an average drift of at least $c\min\{1,\sqrt{\lambda}k/n\}$
over all random rates at distance $k$,
where $c>0$ is a constant.  Considering $\min\{1,\sqrt{\lambda}k/n\}$, 
the minimum is taken on the first argument if $k>n/\sqrt{\lambda}$,
and on the second if $k<n/\sqrt{\lambda}$.

We are interested in the expected time to reduce the \om-value to at most $n/\lambda$.
To ease the application of drift analysis,
we artificially modify the process 
and make it create the optimum
when the state ($\om$-value) is strictly less than $n/\lambda$.
Clearly, the first hitting time of state 
at most $n/\lambda$ does not change 
by this modification.
Applying the variable drift theorem (Theorem~\ref{theo:variable-upper})
with $\xmin=n/\lambda$, $X_0=k\le n/\ln\lambda$ and 
$h(x)=c\min\{1,\sqrt{\lambda}k/n\}$,
the expected number of generations to reach state at most $n/\lambda$
is bounded from above by 
\[
\frac{n/\lambda}{c \sqrt{\lambda} (n/\lambda)/n} + \int_{n/\lambda}^{n/\sqrt{\lambda}} \frac{n}{c\sqrt{\lambda}x}\,\mathrm{d}x + 
\int_{n/\sqrt{\lambda}}^{n/\ln\lambda} \frac{1}{c}\,\mathrm{d}x = O\left(\frac{n}{\sqrt{\lambda}}\right) + 
O\left(\frac{n\log\lambda }{\sqrt{\lambda}}\right) + O\left(\frac{n}{\log \lambda}\right),
\]
which is $O(n/\log \lambda)$.
The overall expected number of generations spent is 
$O(\log n + n/\log \lambda) = O(n/\log \lambda)$
since $\lambda=n^{O(1)}$ by assumption.
\end{proof}

\section{Near Region}
\label{sec:near}

In the near region, we have $k\le n/\lambda$. Hence, the fitness is so low that 
we can expect only a constant number 
of offspring to flip at least one of the remaining one-bits. This assumes constant rate. However, higher rates 
are detrimental since they are more likely to 
destroy the zero-bits of the few individuals flipping one-bits. Hence, we expect the rate to drift towards 
constant values, as shown in the following lemma.


\begin{lemma}
\label{lem:near-region-rate-drift}
Let $k\leq n/\lambda$, $\lambda\geq 45$ and $4\leq r_t\leq\ln(\lambda)/4$.
Then the probability that $r_{t+1}=r_t/2$ is at least $0.5099$.
\end{lemma}

\begin{proof}
To prove the claim we exploit the fact
that only few one-bits are flipped in both subpopulations.
Using $r\coloneqq r_t$, we shall argue as follows. With sufficiently high (constant) probability, (i)~the $2r$-subpopulation contains no individual strictly better than the parent, that is, with fitness less than $k$, and (ii)~all $2r$-offspring with fitness at least~$k$ are identical to the parent. Conditional on this, either the $r/2$-population contains individuals with fitness less than $k$ and the winning individual surely stems from this subpopulation, or the $r/2$-population contains no better offspring. In the latter case, we argue that there are many more individuals with fitness exactly $k$ in the $r/2$-population than in the $2r$-population, which gives a sufficiently high probability for taking the winning individual from this side (as it is chose uniformly at random from all offspring with fitness $k$).
%

Let $N_{r/2}$ and $N_{2r}$ be the number of offspring
that did not flip any zero-bits using rate $r/2$ and $2r$, respectively.
Then $E(N_{r/2})=(\lambda/2)(1-r/(2n))^{n-k}\geq (1-o(1))\lambda e^{-r/2}/2$,
since $k\geq 1$ and
\begin{equation*}
\left(1-\frac{r}{2n}\right)^{n-1}
\geq e^{-\frac{r}{2}}\left(1-\frac{r}{2n}\right)^{\frac{r}{2}-1}
\geq e^{-\frac{r}{2}}\left(1-\frac{c\ln n}{8n}\right)^{\frac{c\ln n}{8}-1}
=(1-o(1))e^{-\frac{r}{2}},
\end{equation*}
where we used that $r\leq\ln\lambda/4$ and $\lambda=n^{O(1)}$,
\ie $\ln\lambda\leq c\ln n$ for some constant $c$.
Using $k\leq n/\lambda$ we get
$E(N_{2r})=(\lambda/2)(1-2r/n)^{n-k}\leq(\lambda/2)e^{-2r(1-1/\lambda)}$.
In fact, we can discriminate $N_{r/2}$ and $N_{2r}$ by using Theorem~\ref{thm:chernoff} in the following way:
we have
\begin{equation*}
\begin{split}
\Prob\left(N_{r/2}\leq\frac{\lambda}{4} e^{-\frac r2}\right)
&\leq\exp\left(-(1-o(1))^3\frac{\lambda}{16} e^{-\frac r2}\right)\\
&\leq\exp\left(-(1-o(1))\frac{1}{16}\lambda^{7/8}\right)\\
&< 0.175,
\end{split}
\end{equation*}
for sufficiently large $n$,
since $r\leq(\ln\lambda)/4$ and $\lambda\geq 45$.
Similarly, we obtain
\begin{equation*}
\Prob\left(N_{2r}\geq\lambda e^{-2r(1-\frac{1}{\lambda})}\right)
\leq \exp\left(-\frac{1}{6}\lambda^{1-\frac 12\left(1-\frac 1\lambda\right)}\right)
< 0.312.
\end{equation*}
Note that $e^{-2r(1-1/\lambda)}<e^{-r/2}/4$ holds for all $r\geq 4$ and $\lambda\geq 45$.
Since the offspring are generated independently,
the events  $N_{r/2}>\lambda e^{-r/2}/4$
and $N_{2r}<\lambda e^{-2r(1-\frac{1}{\lambda})}$
happen together with probability at least $(1-0.175)\cdot (1-0.312)\geq 0.567\eqqcolon 1-p_{\text{err}_1}$. 
Conditioning on this 
and by using a union bound
the probability $p_{\text{err}_2}$ that at least one of the $N_{2r}$ offspring
that do not flip any zero-bits
flips at least one one-bit can be upper bounded by
\begin{equation*}
p_{\text{err}_2}\coloneqq N_{2r}\cdot\frac{2kr}{n}
\leq 2re^{-2r\left(1-\frac 1\lambda\right)}
\leq 0.004
\end{equation*}
using $k/n\leq 1/\lambda$ in the first
and $r\geq 4$ and $\lambda\geq 45$
in the last inequality.
By using a union bound,
we find the probability $p_{\text{err}_3}$
that at least one $2r$-offspring flips exactly one one-bit
and exactly one zero-bit
to be at most
\begin{equation*}
\begin{split}
p_{\text{err}_3}&\coloneqq\frac{\lambda}{2}\frac{2rk}{n}\left(1-\frac{2r}{n}\right)^{k-1}
\frac{2r(n-k)}{n}\left(1-\frac{2r}{n}\right)^{n-k-1}\\
&\leq 2r^2\left(1-\frac{2r}{n}\right)^{n-2}
\leq (1+o(1))2e^{2(\ln(r)-r)}
< (2+o(1))e^{-\frac{6}{5}r} < 0.017,
\end{split}
\end{equation*}
for sufficiently large~$n$, 
using $k/n\leq 1/\lambda$ for the first inequality.
The third inequality is due to
$\ln x-x\leq -(1-e^{-1})x < -(3/5)x$ for all $x>0$ (see Lemma~\ref{lem:elem}.a)
and the last inequality stems from $r\geq 4$.
The second inequality follows from
\begin{equation*}
\left(1-\frac{2r}{n}\right)^{n-2}
\le e^{-\frac{2r}{n}(n-2)} = e^{-2r\left(1-\frac{2}{n}\right)}
=(1+o(1))e^{-2r},
\end{equation*}
using again $r\le \ln\lambda\leq c\ln n$ for some constant $c$.
Let $M_r$ be the number of such offspring. 
Any other fitness-decreasing flip-combinations of zeroes and ones in the $2r$-subpopulation
require an offspring to flip at least two one-bits. The probability that such an offspring is created is at most
\begin{equation*}
p_{\text{err}_4}\coloneqq\frac\lambda 2\binom{k}{2}\left(\frac{2r}{n}\right)^2
\leq\frac{\ln^2\lambda}{16\lambda}
< 0.021,
\end{equation*}
using $k\leq n/\lambda$ and $r\leq\ln\lambda/4$ and the fact that $(\ln^2 x)/ x$ is decreasing
for $x\geq e^2$ and $\lambda \geq 45 > e^2$.

The events $N_{r/2}> \lambda e^{-r/2}/4$, $N_{2r}<\lambda e^{-2r(1-\frac{1}{\lambda})}$, $M_r=0$,
and the event that no fitness-decreasing offspring is created in the $2r$-subpopulation
are sufficient to ensure
that the best individual is either surely from the $r/2$-population or chosen uniformly at random
from the $N_{r/2}+N_{2r}$ offspring.
Conditioning on these events, we have that the probability
that the best offspring is chosen from the $r/2$-population is at least
\begin{equation*}
\frac{N_{r/2}}{N_{r/2}+N_{2r}}
\geq\frac{1}{1+4e^{-r\left(2\left(1-\frac 1\lambda\right)-\frac 12\right)}}
>0.988.
\end{equation*}
Hence, using a union bound for the error probabilities,
the unconditional probability is at least
\begin{equation*}
\frac{\left(1-p_{\text{err}_1}-p_{\text{err}_2}-p_{\text{err}_3}-p_{\text{err}_4}\right)\cdot 0.988+\frac 12}{2}>0.5099.
\end{equation*}

\end{proof}

We note that the restriction $r_t \ge 4$ in the lemma above is not strictly necessary. Also for smaller $r_t$, the probability that the winning individual is chosen from the $r_t/2$-population is by an additive constant larger than $1/2$. Showing this, however, would need additional proof arguments as for smaller $r_t$, the event that both subpopulations contain individuals with fitness $k-1$ becomes more likely. We avoid this additional technicality by only arguing for $r_t \ge 4$, which is enough since any constant $r_t$ is sufficient for the fitness drift we need (since we do not aim at making the leading constant precise). 

In the following proof of the analysis of the near region, we use the above lemma (with quite some additional arguments) to argue that the $r$-value quickly reaches $4$ or less and from then on regularly returns to this region. This allows to argue that in the near region we have a speed-up of a factor of $\Theta(\lambda)$ compared to the \ooea, since every offspring 
only has a probability of $O(1/\lambda)$ of making progress (see also \cite{DoerrKuennemannTCS15,GiessenWittALGO16}).

\begin{theorem}
\label{thm:expected-time-near}
Assume $k\le n/\lambda$ for the current \om-value of the self-adjusting \oplea. Then the expected number 
of generations until the optimum is reached is $O( n\log(n)/\lambda + \log n)$.
\end{theorem}

\begin{proof}
The aim is to estimate the \om-drift at the points in time (generations)~$t$ where $r_t=O(1)$. To bound the 
expected number of generations until the mutation rate has entered this region, 
we basically consider the stochastic process 
$Z_t:=\max\{0,\lfloor\log_2(r_t/a)\rfloor\}$, where $a\coloneqq 4$, 
which is   
the lower bound on $r_t$ from Lemma~\ref{lem:near-region-rate-drift}. However, as we do not have proved 
a drift of $Z_t$ towards smaller values in the region $L\coloneqq (\ln \lambda)/4\le r_t\le 16\ln \lambda \eqqcolon U$ (where $16$ is an 
upper bound on $c_2(k)$ from Lemma~\ref{lem:far-region-drift}), 
we use the potential function
\[
X_t(Z_t)\coloneqq \begin{cases}
Z_t & \text{ if $a\le r_t\le L$}\\
\log_2(L/a) + \sum_{i=1}^{\lceil\log_2(r_t/L)\rceil} 4^{-i} & \text{ if $L<r_t<U$}\\
\log_2(L/a) + \sum_{i=1}^{\log_2(U/L)+1} 4^{-i} + 4^{-\log_2(U/L)-1} \lfloor \log_2(r_t/a) \rfloor & \text{ otherwise}.
\end{cases}
\]
assuming that $L$ and $U$ have been rounded down and up to the closest power of~$2$, respectively. We note that 
$X_t \ge \log_2(r_t/a)-1 = \log_2(r_t/(2a))$ if $r_t\le L$ and $X_t\ge (L/U)^2 \log_2(r_t/(2a))$ in all three cases.
 
The potential function has a slope of $1$ for $a\le r_t\le L$. 
Lemma~\ref{lem:near-region-rate-drift} gives us the 
drift $E(X_t-X_{t+1}\mid X_t;a\le r_t\le L)\ge (0.5099 - 0.4901)/(2a) = \Omega(1)$. The function 
 satisfies 
$X_t(Z_t)-X_t(Z_t-1) \ge 4 (X_t(Z_{t+1})-X_t(Z_t))$ if $L<r_t<U$, which 
corresponds to the region where  
the probability of decreasing $r_t$ by a factor of~$1/2$ has  only be bounded from below 
by $1/4$ due to the random steps. Still, $E(X_t-X_{t+1}\mid X_t; L<r_t<U)\ge (1/4) 4^{-\log_2(U/L)-1} - (3/4) 4^{-\log_2(U/L)}=\Omega(1)$ 
in this region due to the concavity of the potential function. Finally, 
$E(X_t-X_{t+1}\mid X_t;r_t\ge U)=4^{-\log_2(U/L)-1} (1/(2a)) \Omega(1)=\Omega(1)$ by Lemma~\ref{lem:far-region-drift}. Hence, 
altogether $E(X_t-X_{t+1}\mid X_t;r_t\ge a)\ge \kappa$ for some constant~$\kappa>0$. 
 As $X_0 = O(\log n)$, additive drift analysis 
yields an expected number of $O(\log n)$ generations until for the first time  $X_t= 0$ holds, 
corresponding 
to $r_t\le a$. We denote this hitting time by~$T$.

We now consider 
an arbitrary point of time~$t\ge T$. The aim is to show a 
drift on the \om-value, depending on the current \om-value $Y_{t}$, which satisfies 
$Y_t\le n/\lambda$ with probability~$1$. To this end, we will 
use Lemma~\ref{lem:occupation-rt}. We choose $b$ large enough such that $2e^{-2b\cdot \kappa/4}\le 1-\delta$ holds  
for some positive constant~$\delta>0$ and note that $b$ is constant. We consider two 
cases. If $X_t\le b$, which happens with probability at least~$\delta$ according to the lemma, then the bound $X_t \ge (L/U)^2 (\log_2(r_t/(2a)))$ implies 
$r_t \le 2^{(U/L)^2} 2a 2^b$. Hence, we have $r_t=O(1)$ in this case and obtain a probability  of at least 
\[
1-\left(1-\binom{Y_t}{1} \left(\frac{2r_t}{n}\right) \left(1-\frac{2r_t}{n}\right)^{n-1}\right)^\lambda 
\ge 1-\left(1-\Theta\left(\frac{Y_t}{n}\right)\right)^\lambda = \Omega(\lambda Y_t/n)
\]
to improve the \om-value by~$1$, 
using that $Y_t=O(n/\lambda)$ and pessimistically assuming a rate of~$2r_t$ in all offspring. 
If $X_t>b$, we bound the improvement from below by~$0$.
 Using the law of total probability, we obtain 
\[
E(Y_t-Y_{t+1} \mid Y_t; Y_t \le n/\lambda) = \delta \Omega(\lambda Y_t/n) = \Omega(\lambda Y_t/n).
\]
Now a multiplicative drift analysis with respect to the stochastic process on the $Y_t$, more precisely 
Theorem~\ref{thm:multiDrift} using 
$\delta=\Theta(\lambda/n)$ and minimum state~$1$, gives an expected number of  
$O((n/\lambda)\log Y_0)=O(n\log(n)/\lambda)$ generations until the optimum is found. Together 
with the expected number $O(\log n)$ until the $r$-value becomes at most~$a$, this proves the theorem.
\end{proof}

\section{Putting Everything Together}
\label{sec:main-theo}
In this section, we put together the analyses of the different regimes to prove our main result.

\begin{proofof}{Theorem~\ref{thm:main}}
The lower bound actually holds for all unbiased parallel black-box algorithms, as shown in 
\cite{BadkobehPPSN14}. 

We add up the bounds on the expected number 
of generations spent in the three regimes, more precisely we add up the bounds from 
Theorem~\ref{thm:expected-time-far}, Theorem~\ref{thm:expected-time-middle} and Theorem~\ref{thm:expected-time-near}, which gives 
us $O(n/\log \lambda +  n\log(n)/\lambda + \log n )$ generations. Due to our assumption $\lambda=n^{O(1)}$ the bound 
is dominated by $O(n/\log \lambda +  n\log(n)/\lambda )$ as suggested.
\end{proofof}

\begin{proofof}{Lemma~\ref{lem:main-2}}
We basically revisit the regions of different \om-values analyzed in this paper and bound the time spent in these 
regions under the assumption $r=\ln(\lambda)/2$. In the far region, Lemmas~\ref{lem:super-far-region-rate-drift} 
and~\ref{lem:far-region-drift}, applied with this value of~$r$, 
imply a fitness drift of $\Omega(\log(\lambda)/\log(en/k))$ per generation, 
so the expected number of generations spent in the far region is $O(n/\log\lambda)$ as 
computed by variable drift analysis in the proof of Theorem \ref{thm:expected-time-far}.

The middle region is shortened at the lower end. For $k\ge n/\sqrt{\lambda}$, 
Lemma~\ref{lem:middle-region-drift} gives a fitness drift of 
$\Omega(1)$, implying by additive drift analysis $O(n/\log\lambda)$ generations to reduce 
the fitness to at most $n/\sqrt{\lambda}$.

In the near region, which now starts at $n/\sqrt{\lambda}$, we have to argue slightly differently. Note that every offspring 
has a probability of at least $(1-r)^{n} \ge e^{-\ln(\lambda)/2+O(1)}=\Omega(\lambda^{-1/2})$ 
of not flipping a zero-bit. Hence, we expect $\Omega(\sqrt{\lambda})$ such offspring. 
We pessimistically assume that the other individuals do not yield a fitness improvement; 
conceptually, this reduces the population size to $\Omega(\sqrt{\lambda})$ offspring, all of which 
are guaranteed not to flip a zero-bit. Adapting the arguments from the proof of 
Theorem~\ref{thm:expected-time-near}, the probability that at least of one of these individuals flips at least a one-bit 
is at least 
\[
1-\left(1-\binom{Y_t}{1} \left(\frac{r_t}{n}\right) \right)^{\Omega(\sqrt{\lambda})} 
\ge 1-\left(1-\Theta\left(\frac{Y_t}{n}\right)\right)^{\Omega(\sqrt{\lambda})}  = \Omega(\sqrt{\lambda} Y_t/n),
\]
which is a lower bound on the fitness drift. Using the multiplicative drift analysis, the expected number 
of generations in the near region is $O(n\log(n)/\sqrt{\lambda})$. Putting the times for the regions together, we obtain 
the lemma.
\end{proofof}

\section{Experiments}
\label{sec:experiments}

Since our analysis is asymptotic in nature
we performed some elementary experiments in order to see
whether besides the asymptotic runtime improvement (showing an improvement for an unspecified large problem size $n$) we also see an improvement for realistic problem sizes.
For this purpose we implemented the \oplea in C
using the GNU Scientific Library (GSL) for the generation
of pseudo-random numbers.
In this section our performance measure is the runtime,
represented by the number of generations until the optimum is found for the first time.

\begin{figure}
\centering
\includegraphics{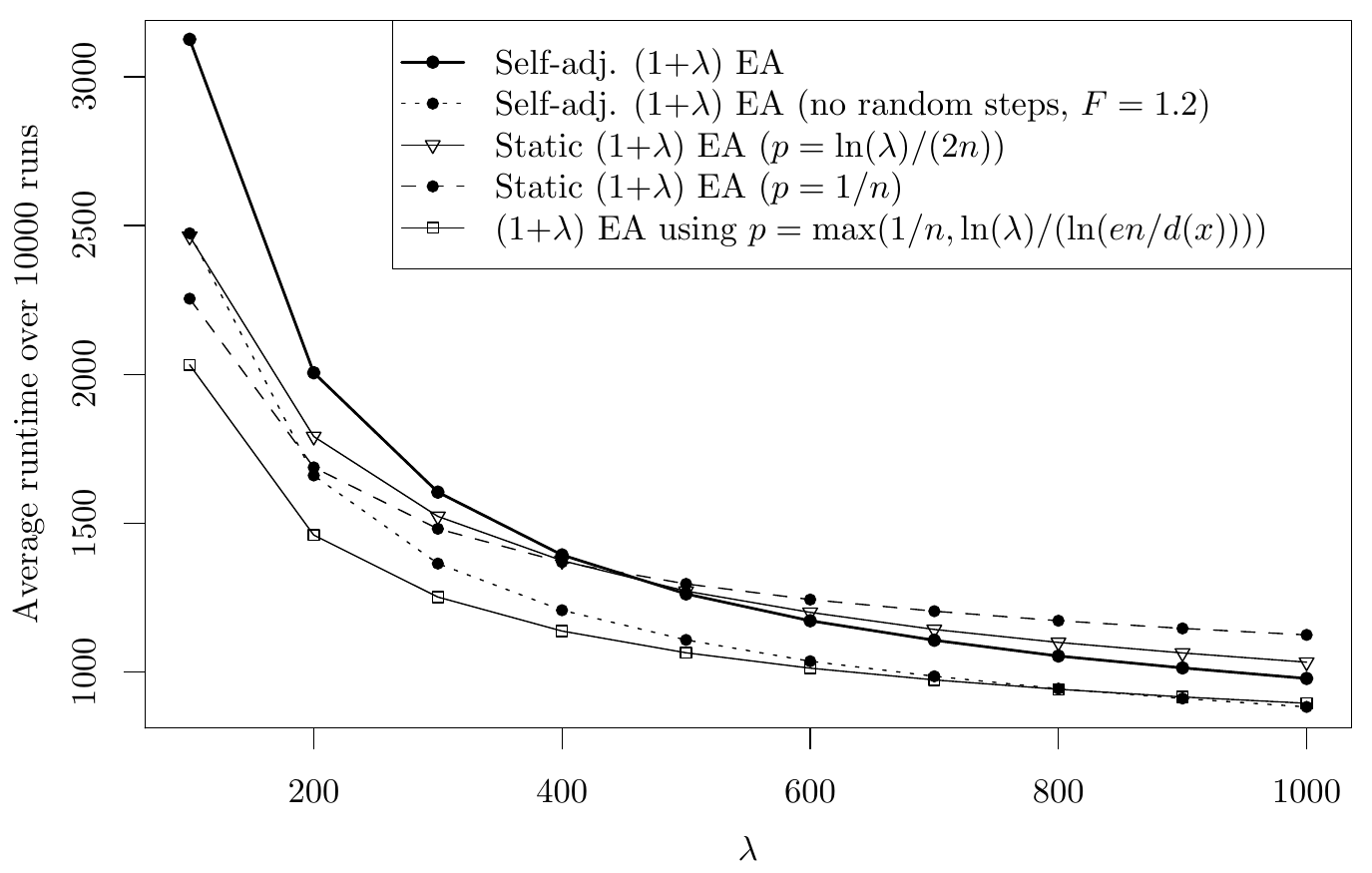}
\includegraphics{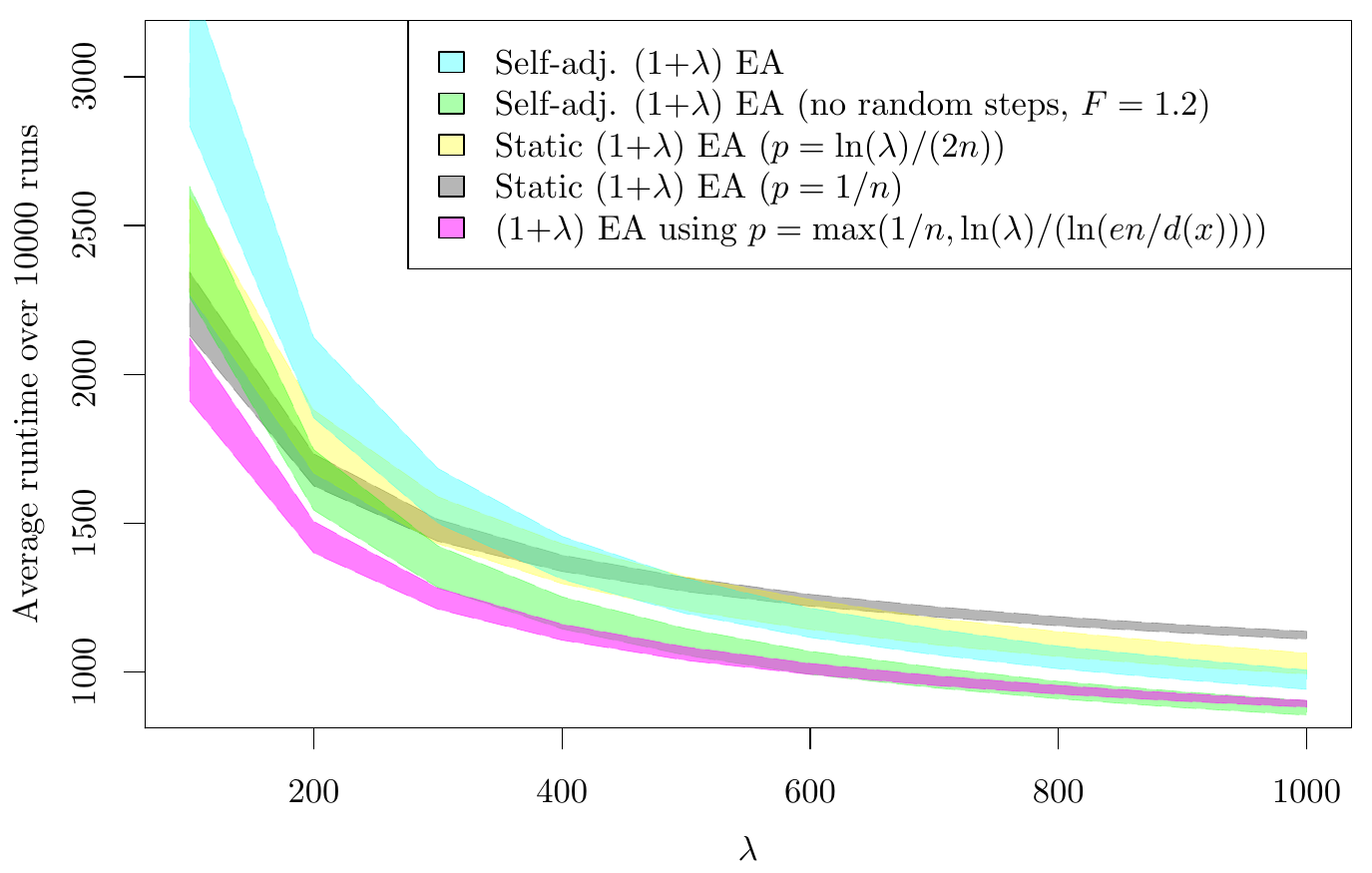}
\caption{Static and Self-adjusting \oplea average runtime comparison ($n = 5000$) on $\om$ and the corresponding interquartile ranges.}
\label{figure:compstatic}
\end{figure}

The first plot in Figure~\ref{figure:compstatic} displays the average runtime over 10000 runs
of the self-adjusting \oplea on \om for $n=5000$ as given in Algorithm~\ref{alg:onelambda}
over $\lambda=100,200,\ldots,1000$; the second plot shows the corresponding interquartile ranges to 
support that the results are statistically significant.
We set the initial mutation rate to $2$,
\ie, the minimum mutation rate the algorithm can attain.
Moreover, the plot displays the average runtimes of the classic \oplea 
using static mutation probabilities of $1/n$ and of $(\ln \lambda)/(2n)$, 
the latter of which  is asymptotically optimal for large~$\lambda$ according to Lemma~\ref{lem:main-2}.

The average runtimes of both algorithms profit
from higher offspring population sizes $\lambda$
leading to lower average runtimes as $\lambda$ increases.
Interestingly, the two static settings of the classic \oplea outperform the self-adjusting \oplea
for small values of $\lambda$ up to $\lambda = 400$.
For higher offspring population sizes
the self-adjusting \oplea outperforms the classic ones,
indicating that the theoretical performance gain of $\ln\ln\lambda$
can in fact be relevant in practice.
Furthermore, we implemented the self-adjusting \oplea without the random steps, that is, when the rate is always adjusted according to how the best offspring are distributed over the two subpopulations. The experiments show that this variant of the self-adjusting \oplea performs generally slightly better on \onemax. Since the \onemax fitness landscape is structurally very simple, this result is not totally surprising. It seems very natural that the fitness of the best of $\lambda/2$ individuals, viewed as a function in the rate, is a unimodal function. In this case, the advantage of random steps to be able to leave local optima of this function is not needed. On the other hand, of course, this observation suggests to try to prove our performance bound rigorously also for the case without random rate adjustments. We currently do not see how to do this.
Lastly, we implemented the \oplea using the fitness-depending mutation rate
$p = \max\{\frac{\ln \lambda}{n \ln(en/d)}, \frac 1n\}$
as presented in~\cite{BadkobehPPSN14}.
The experiments suggest that this scheme outperforms all other variants considered.

Additionally we implemented another variant of the self-adjusting \oplea
using three equally-sized subpopulations
\ie the additional one is using (that is, exploiting) the current mutation rate.
We compared this variant with the self-adjusting \oplea,
both with and without using random steps.
The results are shown in Figure~\ref{figure:threepop}, again with respect to average runtimes 
and interquartile ranges.
The experiments suggest that the variant using three subpopulations
outperforms the self-adjusting \oplea slightly for small population sizes.
For very high population sizes,
using just two subpopulations seems to be a better choice.

\begin{figure}
  \centering
\includegraphics{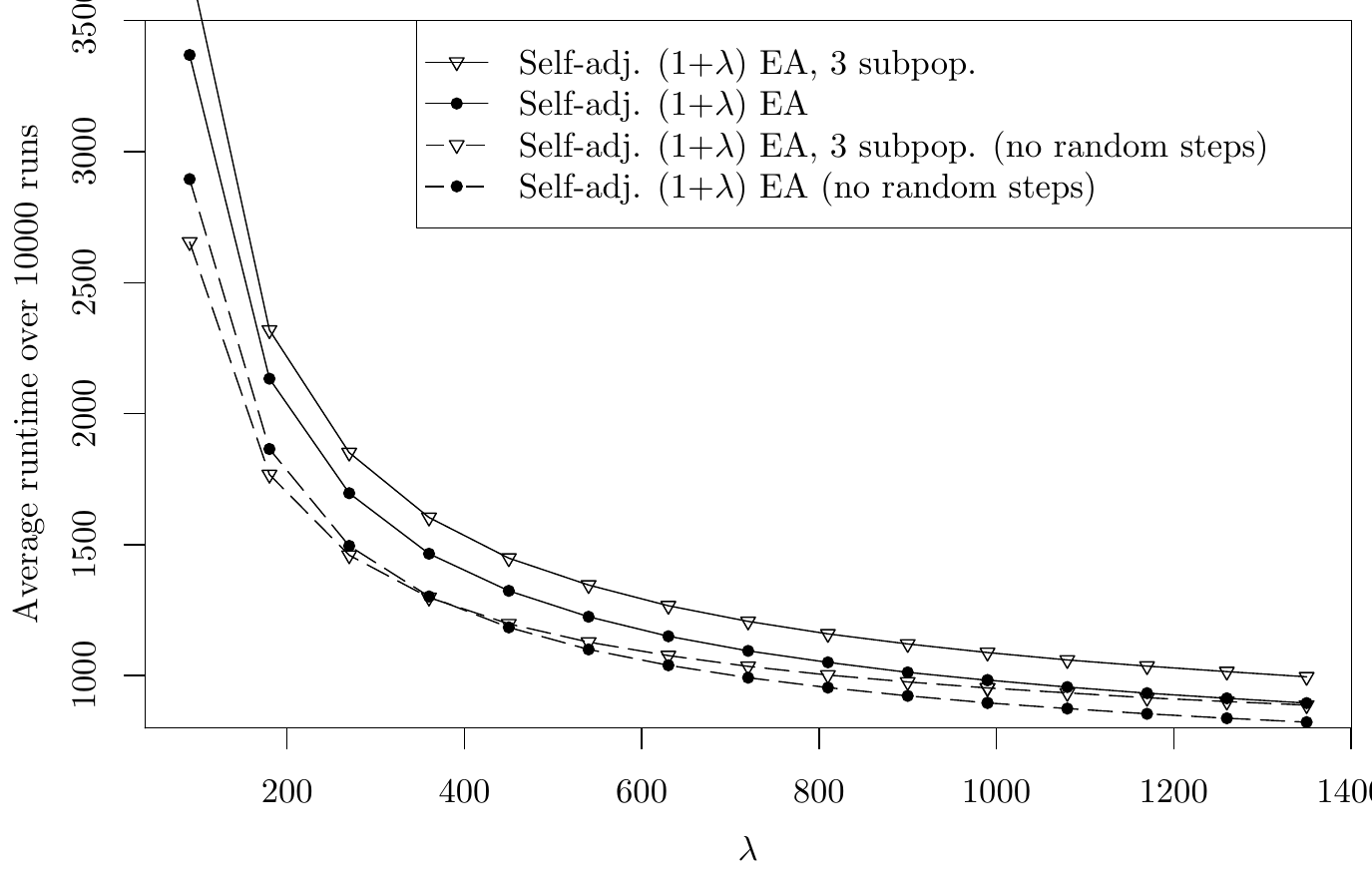}
\includegraphics{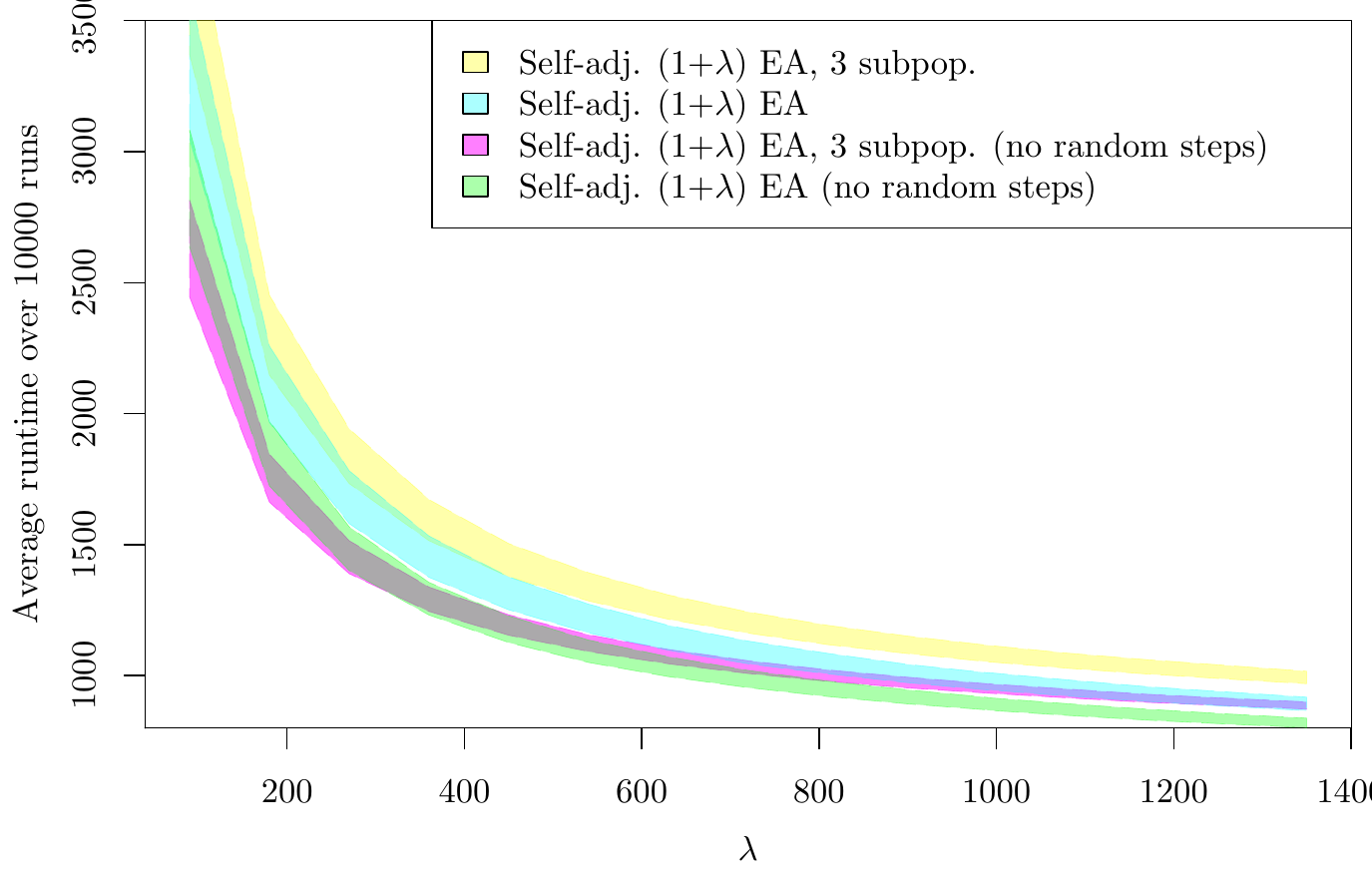}
\caption{Average runtime of the self-adjusting \oplea with two and three subpopulations each with and without random steps on \om ($n=5000$) with corresponding interquartile ranges.}
\label{figure:threepop}
\end{figure}

\begin{figure}
\centering
\includegraphics{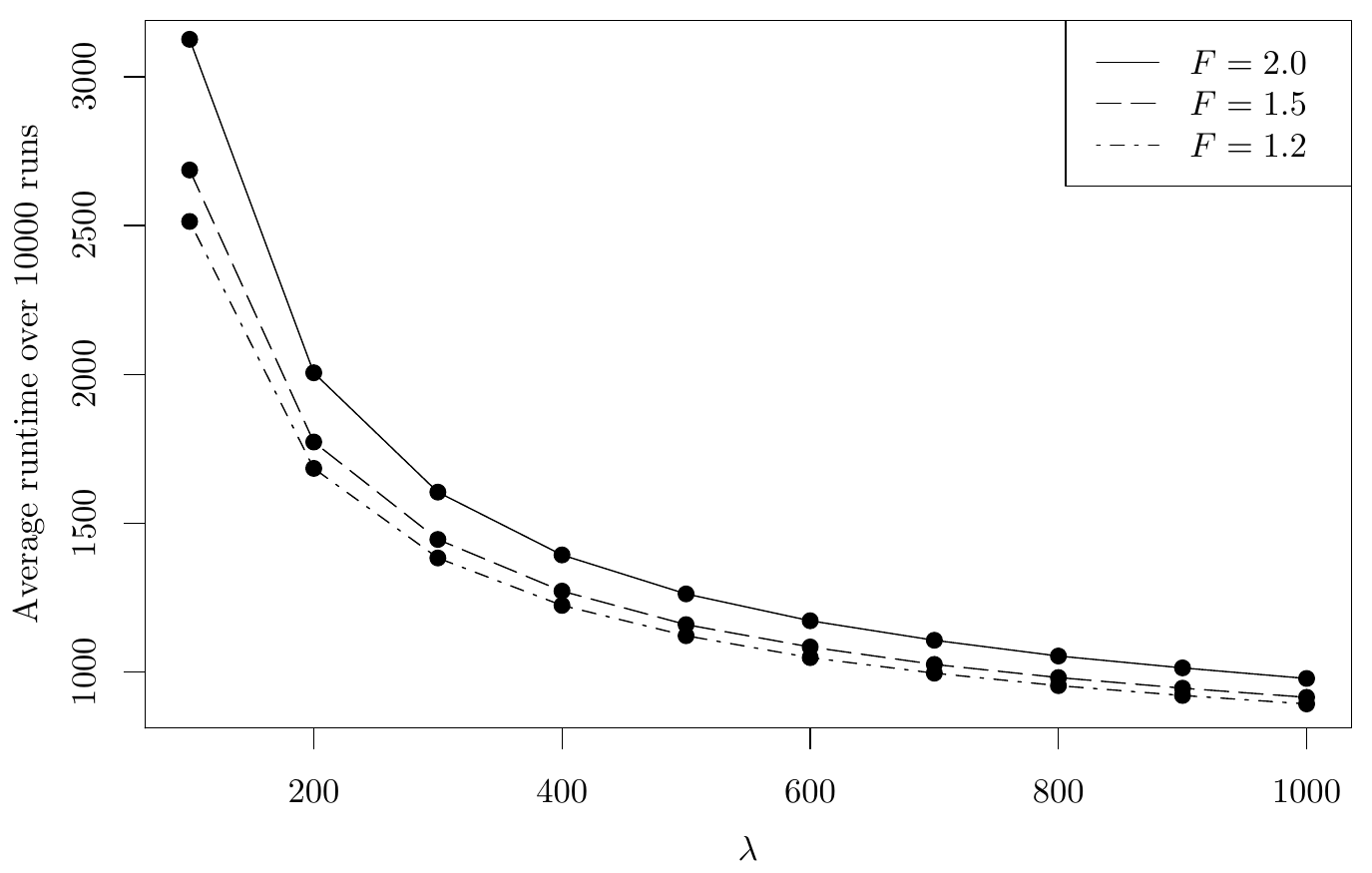}
\includegraphics{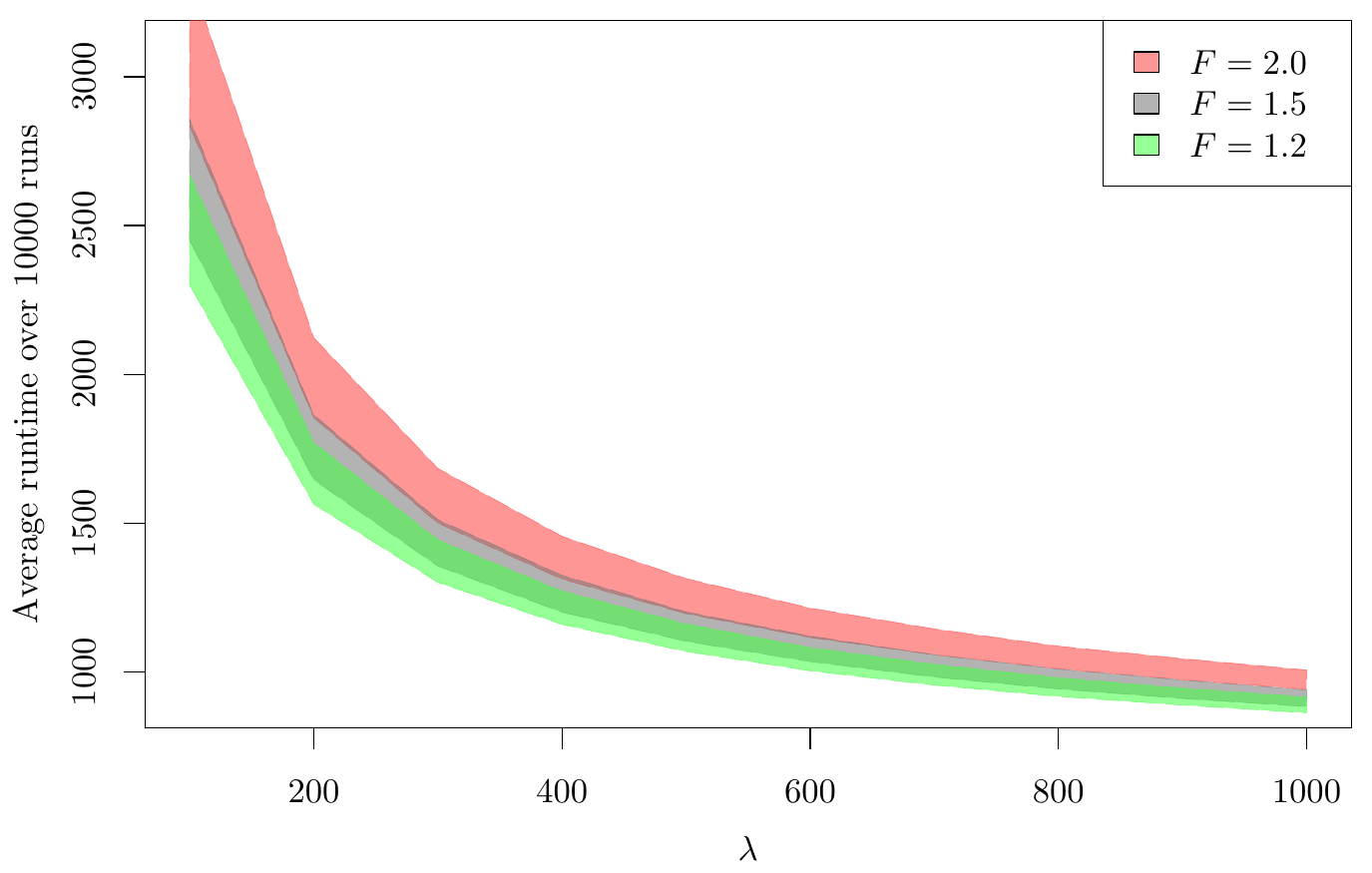}
\caption{Average runtime of the self-adjusting \oplea with different mutation rate update factors $F$ on \om ($n=5000$) with corresponding interquartile ranges}
\label{figure:compF}
\end{figure}

To gain some understanding on how the parameters influence the runtime, we implemented the self-adjusting \oplea
using different mutation rate update factors, that is,
we consider the self-adjusting \oplea as given in Algorithm~\ref{alg:onelambda}
where the mutation rate $r_t$ is increased or decreased by some factor $F$
(instead of the choice $F=2$ made in Algorithm~\ref{alg:onelambda}).
Note that we do not change the rule that we use the rates $r/2$ and $2r$
to create the subpopulations.
Furthermore, after initialization, the algorithm starts with rate $F$
and the rate is capped below by $F$ and above by $1/(2F)$ during the run, accordingly.

The results are shown in Figure~\ref{figure:compF}.
The plot displays the average runtime over 10000 runs
of the self-adjusting \oplea on \om for $n=5000$
over $\lambda=100,200,\ldots,1000$ using the update factors $F = 2.0, 1.5, 1.01$.
The plot suggests that lower values of $F$ yield a better performance. This result is not immediately obvious. Clearly, a large factor $F$ implies that the rate changes a lot from generation to generation (namely by a factor of $F$). These changes prevent the algorithm from using a very good rate for several iterations in a row. On the other hand, a small value for $F$ implies that it takes longer to adjust the rate to value that is far from the current one. We also performed experiments for $F>2$ and observed even worse runtimes than for $F=2.0$. The figure does not display the outcomes of these additional experiments since this would impair the readability of the diagram.

\begin{figure}
\begin{center}
\includegraphics{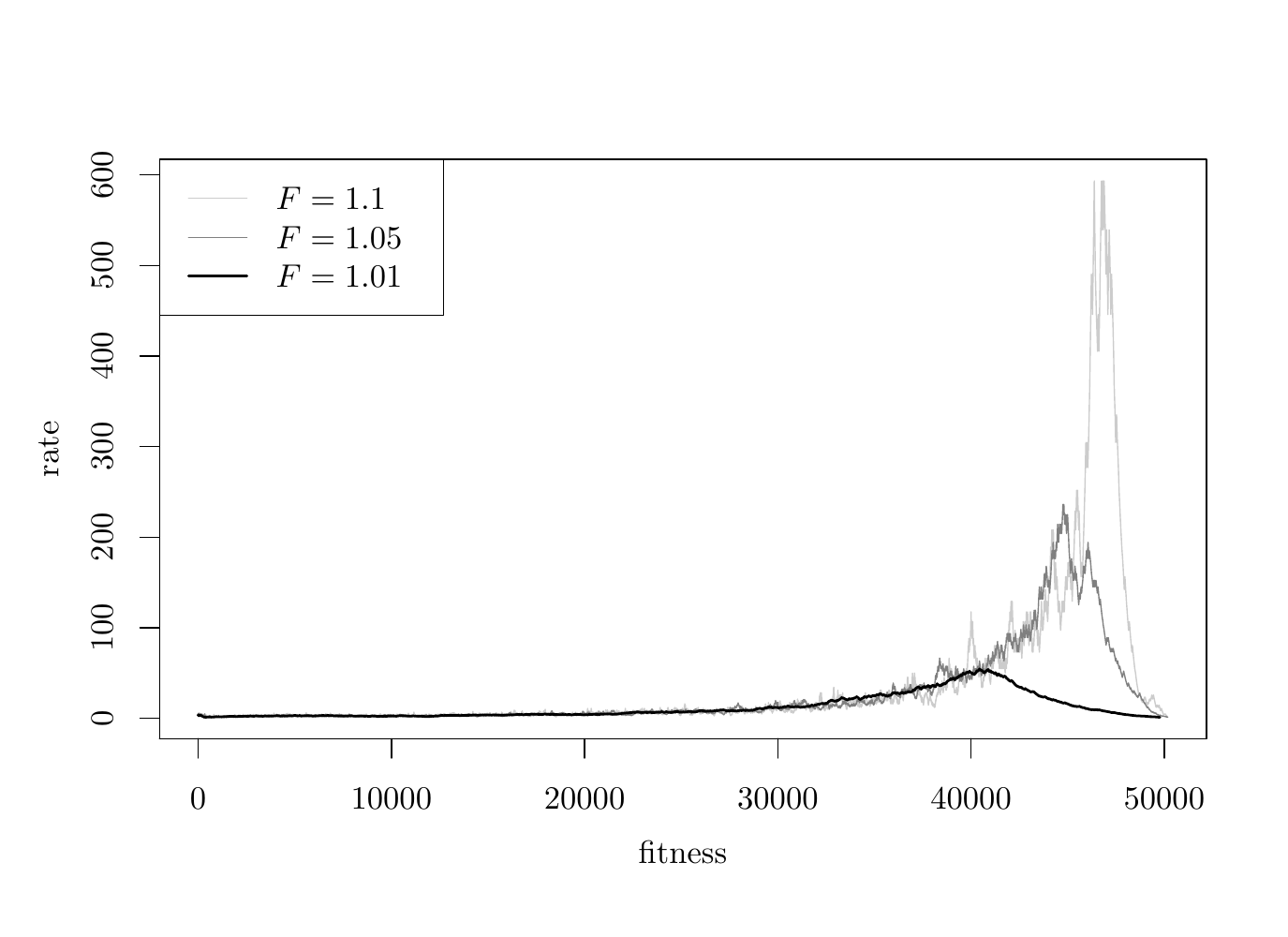}
\end{center}
\caption{Development of the rate over the fitness of three example runs of the self-adjusting \oplea on \om ($n=100000$, $\lambda=1000$), using different factors $F$}
\label{figure:nontrivial}
\end{figure}

Finally, to illustrate the nontrivial development of the rate
during a run of the algorithm
we plotted the rate of three single runs of the self-adjusting \oplea
using different factors $F$ over the fitness in Figure~\ref{figure:nontrivial}.
Since the algorithm initialized with the rate $F$,
the rate increases after initialization
and decreases again with decreasing fitness-distance to the optimum.
The plot suggests that for higher values of $F$ the rate is more unsteady due to the greater impact of the rate adjustments
while smaller rate updates yield a more stable development of the rate.
Interestingly, for all three values of $F$,
the rates seem to correspond to the same rate
after the initial increasing phase from $F$.
Note that this illustration does not indicate the actual runtime.
In fact, the specific runtimes are $19766$ for $F=1.01$,
$19085$ for $F=1.05$ and $19857$ for $F=1.1$.
A similar, more pronounced behaviour can be seen for $F=2.0$;
we chose these particular values of $F$ for illustrative purposes
since for $F=2.0$ the variance in the rate can be visually confusing
for the reasons given above.

While we would draw from this experiment the conclusion that a smaller choice of $F$ is preferable in a practical application of our algorithm, the influence of the parameter on the runtime is not very large. So it might not be worth optimizing it and rather view Algorithm~\ref{alg:onelambda} as a parameter-less algorithm.

\section{Conclusions}
\label{sec:conclusion}

We proposed and analyzed a new simple self-adjusting mutation scheme for the \oplea. It consists of creating half the offspring with a slightly larger and the rest with a slightly smaller mutation rate. Based on the success of the subpopulations, the mutation rate is adjusted. This simple scheme overcomes difficulties of previous self-adjusting choices, e.g., the careful choice of the exploration-exploitation balance and the forgetting rate in the learning scheme of~\cite{DoerrYangPPSN16}.

We proved rigorously that this self-adjusting \oplea optimizes the \onemax test function in an expected number of $O(n \lambda / \log \lambda + n \log n)$ fitness evaluations. This matches the runtime shown in~\cite{BadkobehPPSN14} for a careful fitness-dependent choice of the mutation rate, which was also shown to be asymptotically optimal among all $\lambda$-parallel black-box optimization algorithms. Hence our runtime result indicates that the self-adjusting mechanism developed in this work is able to find very good mutation rates. To the best of our knowledge, this is the first time that a self-adjusting choice of the mutation rate speeds up a mutation-based algorithm on the \onemax test function by more than a constant factor.

The main technical challenge in this work is to analyze the quality of the best offspring. In contrast to most previous runtime analyses, where only the asymptotic order of the fitness gain was relevant, we needed a much higher degree of precision as we needed to make statements about in which sub-population the best offspring is, or, in case of multiple best offspring, how they are distributed over the two subpopulations. Note that the quality of the best offspring is not as strongly concentrated around its expectation 
as, e.g., the average quality.

As a side-result of our analyses, 
we have observed that using  a fixed rate of 
$r=\ln(\lambda)/2$ gives the 
 bound $O(n/\log \lambda +  n\log(n)/\lambda^{1/2} )$, which is also asymptotically optimal 
unless $\lambda$ is  small. However, this setting is far off the usual constant choice of~$r$. It is the first time that a significantly larger mutation rate was shown to be useful in a simple mutation-based algorithm for a simple fitness landscape. Previously, it was only observed that larger mutation rates can be helpful to leave local optima~\cite{DoerrLMN17}.

From this work, a number of open problems arise. A technical challenge is to prove that our algorithm also without the random rate adjustments performs well. This requires an even more precise analysis of the qualities of the offspring in the two sub-populations, for which we currently do not have the methods. From the view-point of understanding the mutation rate for population-based algorithms, two interesting questions are (i)~to what extent our observation that larger mutation rates are beneficial for the \oplea on \onemax generalizes to other algorithms and problems, and (ii)~for which other problems our self-adjusting choice of the mutation rate gives an improvement over the classic choice of $1/n$ or other static choices.

\paragraph{Acknowledgments}

This work was supported by a public grant as part of the
Investissement d'avenir project, reference ANR-11-LABX-0056-LMH,
LabEx LMH, and by a grant by the Danish Council for Independent Research  (DFF-FNU 4002--00542).

\bibliographystyle{plainnat}

\begin{thebibliography}{45}
\providecommand{\natexlab}[1]{#1}
\providecommand{\url}[1]{\texttt{#1}}
\expandafter\ifx\csname urlstyle\endcsname\relax
  \providecommand{\doi}[1]{doi: #1}\else
  \providecommand{\doi}{doi: \begingroup \urlstyle{rm}\Url}\fi

\bibitem[Alanazi and Lehre(2014)]{AlanaziL14}
Fawaz Alanazi and Per~Kristian Lehre.
\newblock Runtime analysis of selection hyper-heuristics with classical
  learning mechanisms.
\newblock In \emph{Proc.\ of CEC '14}, pages 2515--2523. IEEE, 2014.

\bibitem[Badkobeh et~al.(2014)Badkobeh, Lehre, and Sudholt]{BadkobehPPSN14}
Golnaz Badkobeh, Per~Kristian Lehre, and Dirk Sudholt.
\newblock Unbiased black-box complexity of parallel search.
\newblock In \emph{Proc.\ of PPSN~'14}, pages 892--901. Springer, 2014.

\bibitem[B{\"{o}}ttcher et~al.(2010)B{\"{o}}ttcher, Doerr, and
  Neumann]{BottcherDN10}
S{\"{u}}ntje B{\"{o}}ttcher, Benjamin Doerr, and Frank Neumann.
\newblock Optimal fixed and adaptive mutation rates for the {LeadingOnes}
  problem.
\newblock In \emph{Proc.\ of PPSN~'10}, pages 1--10. Springer, 2010.

\bibitem[Buzdalov and Doerr(2017)]{BuzdalovD17}
Maxim Buzdalov and Benjamin Doerr.
\newblock Runtime analysis of the $(1+(\lambda,\lambda))$ genetic algorithm on
  random satisfiable 3-{CNF} formulas.
\newblock In \emph{Proc.\ of GECCO '17}, pages 1343--1350. {ACM}, 2017.

\bibitem[Cathabard et~al.(2011)Cathabard, Lehre, and Yao]{CathabardLY11}
Stephan Cathabard, Per~Kristian Lehre, and Xin Yao.
\newblock Non-uniform mutation rates for problems with unknown solution
  lengths.
\newblock In \emph{Proc.\ of FOGA '11}, pages 173--180. ACM, 2011.

\bibitem[Cervantes and Stephens(2008)]{CervantesS08}
Jorge Cervantes and Christopher~R. Stephens.
\newblock Rank based variation operators for genetic algorithms.
\newblock In \emph{Proc.\ of GECCO '08}, pages 905--912. {ACM}, 2008.

\bibitem[Dang and Lehre(2016)]{DangL16}
Duc{-}Cuong Dang and Per~Kristian Lehre.
\newblock Self-adaptation of mutation rates in non-elitist populations.
\newblock In \emph{Proc.\ of PPSN~'16}, pages 803--813. Springer, 2016.

\bibitem[Dietzfelbinger et~al.(2010)Dietzfelbinger, Rowe, Wegener, and
  Woelfel]{DietzfelbingerRWW10}
Martin Dietzfelbinger, Jonathan~E. Rowe, Ingo Wegener, and Philipp Woelfel.
\newblock Tight bounds for blind search on the integers and the reals.
\newblock \emph{Combinatorics, Probability {\&} Computing}, 19:\penalty0
  711--728, 2010.

\bibitem[Doerr(2011)]{AugerDoerr11}
Benjamin Doerr.
\newblock Analyzing randomized search heuristics: tools from probability
  theory.
\newblock In Anne Auger and Benjamin Doerr, editors, \emph{Theory of Randomized
  Search Heuristics}, pages 1--20. World Scientific Publishing, 2011.

\bibitem[Doerr(2016)]{Doerr16}
Benjamin Doerr.
\newblock Optimal parameter settings for the $(1+(\lambda, \lambda))$ genetic
  algorithm.
\newblock In \emph{Proc.\ of GECCO~'16}, pages 1107--1114. {ACM}, 2016.

\bibitem[Doerr(2018)]{Doerr18exceedexp}
Benjamin Doerr.
\newblock An elementary analysis of the probability that a binomial random
  variable exceeds its expectation.
\newblock \emph{Probability and Statistics Letters}, 2018.
\newblock To appear. Preprint \emph{arXiv:1712.00519}.

\bibitem[Doerr and Doerr(2015)]{DoerrD15}
Benjamin Doerr and Carola Doerr.
\newblock Optimal parameter choices through self-adjustment: applying the
  1/5-th rule in discrete settings.
\newblock In \emph{Proc.\ of GECCO~'15}, pages 1335--1342. {ACM}, 2015.

\bibitem[Doerr and K{\"{u}}nnemann(2015)]{DoerrKuennemannTCS15}
Benjamin Doerr and Marvin K{\"{u}}nnemann.
\newblock Optimizing linear functions with the (1+{\(\lambda\)}) evolutionary
  algorithm -- different asymptotic runtimes for different instances.
\newblock \emph{Theoretical Computer Science}, 561:\penalty0 3--23, 2015.

\bibitem[Doerr et~al.(2012)Doerr, Johannsen, and
  Winzen]{DoerrJohannsenWinzenALGO12}
Benjamin Doerr, Daniel Johannsen, and Carola Winzen.
\newblock Multiplicative drift analysis.
\newblock \emph{Algorithmica}, 64:\penalty0 673--697, 2012.

\bibitem[Doerr et~al.(2015{\natexlab{a}})Doerr, Doerr, and
  Ebel]{DoerrEbelTCS15}
Benjamin Doerr, Carola Doerr, and Franziska Ebel.
\newblock From black-box complexity to designing new genetic algorithms.
\newblock \emph{Theoretical Computer Science}, 567:\penalty0 87--104,
  2015{\natexlab{a}}.

\bibitem[Doerr et~al.(2015{\natexlab{b}})Doerr, Doerr, and
  K{\"{o}}tzing]{DoerrDK15}
Benjamin Doerr, Carola Doerr, and Timo K{\"{o}}tzing.
\newblock Solving problems with unknown solution length at (almost) no extra
  cost.
\newblock In \emph{Proc.\ of GECCO '15}, pages 831--838. {ACM},
  2015{\natexlab{b}}.

\bibitem[Doerr et~al.(2016{\natexlab{a}})Doerr, Doerr, and
  K{\"{o}}tzing]{DoerrDK16}
Benjamin Doerr, Carola Doerr, and Timo K{\"{o}}tzing.
\newblock The right mutation strength for multi-valued decision variables.
\newblock In \emph{Proc.~of GECCO '16}, pages 1115--1122. {ACM},
  2016{\natexlab{a}}.

\bibitem[Doerr et~al.(2016{\natexlab{b}})Doerr, Doerr, and
  K{\"{o}}tzing]{DoerrKoetzingPPSN16}
Benjamin Doerr, Carola Doerr, and Timo K{\"{o}}tzing.
\newblock Provably optimal self-adjusting step sizes for multi-valued decision
  variables.
\newblock In \emph{Proc.\ of PPSN~'16}, pages 782--791. Springer,
  2016{\natexlab{b}}.

\bibitem[Doerr et~al.(2016{\natexlab{c}})Doerr, Doerr, and
  Yang]{DoerrYangGECCO16}
Benjamin Doerr, Carola Doerr, and Jing Yang.
\newblock Optimal parameter choices via precise black-box analysis.
\newblock In \emph{Proc.\ of GECCO~'16}, pages 1123--1130. ACM,
  2016{\natexlab{c}}.

\bibitem[Doerr et~al.(2016{\natexlab{d}})Doerr, Doerr, and
  Yang]{DoerrYangPPSN16}
Benjamin Doerr, Carola Doerr, and Jing Yang.
\newblock $k$-bit mutation with self-adjusting $k$ outperforms standard bit
  mutation.
\newblock In \emph{Proc.\ of PPSN~'16}, pages 824--834. Springer,
  2016{\natexlab{d}}.

\bibitem[Doerr et~al.(2017{\natexlab{a}})Doerr, Doerr, and
  K{\"{o}}tzing]{DoerrDK17}
Benjamin Doerr, Carola Doerr, and Timo K{\"{o}}tzing.
\newblock Unknown solution length problems with no asymptotically optimal run
  time.
\newblock In \emph{Proc.\ of GECCO '17}, pages 1367--1374. {ACM},
  2017{\natexlab{a}}.

\bibitem[Doerr et~al.(2017{\natexlab{b}})Doerr, Gie{\ss}en, Witt, and
  Yang]{DoerrGWYGECCO17}
Benjamin Doerr, Christian Gie{\ss}en, Carsten Witt, and Jing Yang.
\newblock The (1+$\lambda$)~evolutionary algorithm with self-adjusting mutation
  rate.
\newblock In \emph{Proc.\ of GECCO '17}, pages 1351--1358. {ACM},
  2017{\natexlab{b}}.

\bibitem[Doerr et~al.(2017{\natexlab{c}})Doerr, Le, Makhmara, and
  Nguyen]{DoerrLMN17}
Benjamin Doerr, Huu~Phuoc Le, R\'egis Makhmara, and Ta~Duy Nguyen.
\newblock Fast genetic algorithms.
\newblock In \emph{Proc.~of GECCO '17}, pages 777--784. {ACM},
  2017{\natexlab{c}}.

\bibitem[Doerr et~al.(2018{\natexlab{a}})Doerr, Lissovoi, Oliveto, and
  Warwicker]{DoerrLOW18}
Benjamin Doerr, Andrei Lissovoi, Pietro~S. Oliveto, and John~Alasdair
  Warwicker.
\newblock On the runtime analysis of selection hyper-heuristics with adaptive
  learning periods.
\newblock In \emph{Proc.\ of GECCO '18}. ACM, 2018{\natexlab{a}}.
\newblock To appear.

\bibitem[Doerr et~al.(2018{\natexlab{b}})Doerr, Witt, and Yang]{DoerrWY18}
Benjamin Doerr, Carsten Witt, and Jing Yang.
\newblock Runtime analysis for self-adaptive mutation rates.
\newblock In \emph{Proc.\ GECCO '18}. ACM, 2018{\natexlab{b}}.
\newblock To appear.

\bibitem[Eiben et~al.(1999)Eiben, Hinterding, and Michalewicz]{EibenHM99}
Agoston~Endre Eiben, Robert Hinterding, and Zbigniew Michalewicz.
\newblock Parameter control in evolutionary algorithms.
\newblock \emph{IEEE Transactions on Evolutionary Computation}, 3:\penalty0
  124--141, 1999.

\bibitem[Gießen and Witt(2017)]{GiessenWittALGO16}
Christian Gießen and Carsten Witt.
\newblock The interplay of population size and mutation probability in the
  (1+$\lambda$) {EA} on {OneMax}.
\newblock \emph{Algorithmica}, 78:\penalty0 587–609, 2017.

\bibitem[Giel and Wegener(2003)]{GielW03}
Oliver Giel and Ingo Wegener.
\newblock Evolutionary algorithms and the maximum matching problem.
\newblock In \emph{Proc.\ of STACS '03}, pages 415--426. Springer, 2003.

\bibitem[Jansen and Wegener(2006)]{JansenW06}
Thomas Jansen and Ingo Wegener.
\newblock On the analysis of a dynamic evolutionary algorithm.
\newblock \emph{Journal of Discrete Algorithms}, 4:\penalty0 181--199, 2006.

\bibitem[Jansen et~al.(2005)Jansen, Jong, and Wegener]{JansenDW05}
Thomas Jansen, Kenneth A.~De Jong, and Ingo Wegener.
\newblock On the choice of the offspring population size in evolutionary
  algorithms.
\newblock \emph{Evolutionary Computation}, 13:\penalty0 413--440, 2005.

\bibitem[Johannsen(2010)]{Johannsen10}
Daniel Johannsen.
\newblock \emph{Random combinatorial structures and randomized search
  heuristics}.
\newblock PhD thesis, Saarland University, 2010.

\bibitem[Kaas and Buhrman(1980)]{Kaas80}
Rob Kaas and Jan~M. Buhrman.
\newblock Mean, median and mode in binomial distributions.
\newblock \emph{Statistica Neerlandica}, 34:\penalty0 13--18, 1980.

\bibitem[K{\"{o}}tzing et~al.(2015)K{\"{o}}tzing, Lissovoi, and
  Witt]{KotzingLissWittFOGA15}
Timo K{\"{o}}tzing, Andrei Lissovoi, and Carsten Witt.
\newblock {(1+1)} {EA} on generalized dynamic {OneMax}.
\newblock In \emph{Proc.\ of FOGA~'15}, pages 40--51. {ACM}, 2015.

\bibitem[L{\"{a}}ssig and Sudholt(2011)]{LaessigFOGA11}
J{\"{o}}rg L{\"{a}}ssig and Dirk Sudholt.
\newblock Adaptive population models for offspring populations and parallel
  evolutionary algorithms.
\newblock In \emph{Proc.\ of FOGA~'11}, pages 181--192. ACM, 2011.

\bibitem[Lehre and \"{O}zcan(2013)]{LehreO13}
Per~Kristian Lehre and Ender \"{O}zcan.
\newblock A runtime analysis of simple hyper-heuristics: To mix or not to mix
  operators.
\newblock In \emph{Proc.\ of FOGA '13}, pages 97--104. ACM, 2013.

\bibitem[Lehre and Witt(2014)]{LehreWittISAAC14}
Per~Kristian Lehre and Carsten Witt.
\newblock Concentrated hitting times of randomized search heuristics with
  variable drift.
\newblock In \emph{Proc.\ of ISAAC~'14}, pages 686--697. Springer, 2014.

\bibitem[Lissovoi et~al.(2017)Lissovoi, Oliveto, and Warwicker]{LissovoiOW17}
Andrei Lissovoi, Pietro~S. Oliveto, and John~Alasdair Warwicker.
\newblock On the runtime analysis of generalised selection hyper-heuristics for
  pseudo-{B}oolean optimisation.
\newblock In \emph{Proc.\ of GECCO '17}, pages 849--856. ACM, 2017.

\bibitem[Mitavskiy et~al.(2009)Mitavskiy, Rowe, and
  Cannings]{MitavskiyVariable}
Boris Mitavskiy, Jonathan~E. Rowe, and Chris Cannings.
\newblock Theoretical analysis of local search strategies to optimize network
  communication subject to preserving the total number of links.
\newblock \emph{International Journal of Intelligent Computing and
  Cybernetics}, 2:\penalty0 243--284, 2009.

\bibitem[Neumann and Wegener(2007)]{NeumannW07}
Frank Neumann and Ingo Wegener.
\newblock Randomized local search, evolutionary algorithms, and the minimum
  spanning tree problem.
\newblock \emph{Theoretical Computer Science}, 378:\penalty0 32--40, 2007.

\bibitem[Oliveto et~al.(2009)Oliveto, Lehre, and Neumann]{OlivetoCEC09}
Pietro~Simone Oliveto, Per~Kristian Lehre, and Frank Neumann.
\newblock Theoretical analysis of rank-based mutation - combining exploration
  and exploitation.
\newblock In \emph{Proc.\ of CEC~'09}, pages 1455--1462. IEEE, 2009.

\bibitem[Qian et~al.(2016)Qian, Tang, and Zhou]{QianTZ16}
Chao Qian, Ke~Tang, and Zhi{-}Hua Zhou.
\newblock Selection hyper-heuristics can provably be helpful in evolutionary
  multi-objective optimization.
\newblock In \emph{Proc.\ of {PPSN} '16}, pages 835--846. Springer, 2016.

\bibitem[Robbins(1955)]{Robbins55}
Herbert Robbins.
\newblock A remark on {S}tirling's formula.
\newblock \emph{The American Mathematical Monthly}, 62:\penalty0 26--29, 1955.

\bibitem[Wegener(2005)]{Wegener05}
Ingo Wegener.
\newblock Simulated annealing beats {M}etropolis in combinatorial optimization.
\newblock In \emph{Proc. of ICALP '05}, pages 589--601. Springer, 2005.

\bibitem[Zarges(2008)]{Zarges08}
Christine Zarges.
\newblock Rigorous runtime analysis of inversely fitness proportional mutation
  rates.
\newblock In \emph{Proc. of PPSN '08}, pages 112--122. Springer, 2008.

\bibitem[Zarges(2009)]{Zarges09}
Christine Zarges.
\newblock On the utility of the population size for inversely fitness
  proportional mutation rates.
\newblock In \emph{Proc. of FOGA '09}, pages 39--46. {ACM}, 2009.

\end{thebibliography}

\bigskip

\end{document}